\documentclass[final,12pt]{clear2024} % Include author names

\usepackage{times}
\usepackage[utf8]{inputenc}
\usepackage[titletoc]{appendix}
\usepackage{multicol}
\usepackage[utf8]{inputenc}
\usepackage{bm}
\usepackage{titlesec}
\usepackage{marginnote}
\usepackage{mathtools}
\usepackage{algorithm}
\usepackage{algpseudocode}
\usepackage{caption}
\usepackage{subcaption}
\usepackage{booktabs}
\usepackage[noabbrev]{cleveref}
\usepackage[framemethod=TikZ]{mdframed}
\usepackage{thmtools,thm-restate}
\usepackage[export]{adjustbox}
\usepackage{rotating}
\usepackage{xspace}
\usepackage{wrapfig}

\newcommand{\spara}[1]{\smallskip\noindent{\textbf{#1}}}

\newenvironment{squishlist}
{\begin{list}{$\bullet$}
  {\setlength{\itemsep}{0pt}
   \setlength{\parsep}{3pt}
   \setlength{\topsep}{3pt}
   \setlength{\partopsep}{0pt}
   \setlength{\leftmargin}{2em}
   \setlength{\labelwidth}{1.5em}
   \setlength{\labelsep}{0.5em} } }
{\end{list}}

\newenvironment{squishlisttight}
{\begin{list}{$\bullet$}
  {\setlength{\itemsep}{0pt}
   \setlength{\parsep}{0pt}
   \setlength{\topsep}{0pt}
   \setlength{\partopsep}{0pt}
   \setlength{\leftmargin}{2em}
   \setlength{\labelwidth}{1.5em}
   \setlength{\labelsep}{0.5em} } }
 {\end{list}}

\newenvironment{code}
{\bgroup\leftskip 20pt\rightskip 20pt \small\noindent{\bfseries
Code:} \ignorespaces}%
{\par\egroup\vskip 0.25ex}

\newtheorem{assumption}{Assumption}

\newtheorem{problem}{Problem}

\newcommand{\nat}{\mathbb{N}\xspace}
\newcommand{\real}{\mathbb{R}\xspace}

\newcommand{\multiblock}{\mathsf{B}\xspace}

\newcommand{\vertexset}{\ensuremath{\mathcal{V}}\xspace}

\newcommand{\causalset}{\ensuremath{\mathcal{E}^s}\xspace}
\newcommand{\causalsetj}{\ensuremath{\mathcal{E}^{s,j}}\xspace}
\newcommand{\idiosyncraticset}{\ensuremath{\mathcal{I}_p^s}\xspace}

\newcommand{\backboneset}{\ensuremath{\mathcal{A}_p}\xspace}
\newcommand{\estimatedbackboneset}{\ensuremath{\widehat{\mathcal{A}}_p}\xspace}
\newcommand{\estimatedbackbonesetj}{\ensuremath{\widehat{\mathcal{A}}_p^j}\xspace}
\newcommand{\candidateuniverse}{\ensuremath{\ensuremath{\mathcal{U}_p}}\xspace}

\newcommand{\candidateedge}{\ensuremath{u_{lm}}\xspace}
\newcommand{\estimatedindividualcausalgraph}{\ensuremath{\widehat{G}^s}\xspace}
\newcommand{\estimatedcausalset}{\ensuremath{\widehat{\mathcal{E}}}\xspace}
\newcommand{\estimatedindividualcausalset}{\ensuremath{\estimatedcausalset^s}\xspace}

\newcommand{\candidatecausalgraph}{\ensuremath{\widehat{G}^s_{\estimatedbackboneset}}\xspace}

\newcommand{\data}{\ensuremath{\mathbf{X}^s}\xspace}
\newcommand{\signal}{\ensuremath{\mathbf{x}^s}\xspace}
\newcommand{\sample}{\ensuremath{\mathbf{x}^s[n]}\xspace}
\newcommand{\multidata}{\ensuremath{\widetilde{\mathbf{X}}^s}\xspace}
\newcommand{\multidataT}{\ensuremath{\widetilde{\mathbf{X}}^{s^\top}}\xspace}
\newcommand{\multisample}{\ensuremath{\widetilde{\mathbf{x}}^s[n]}\xspace}
\newcommand{\multisignal}{\ensuremath{\tilde{\mathbf{x}}^{s,j}}\xspace}
\newcommand{\causalstructure}{\ensuremath{\mathbf{C}^s}\xspace}
\newcommand{\causalstructureT}{\ensuremath{\mathbf{C}^{s^\top}}\xspace}
\newcommand{\estimatedcausalstructure}{\ensuremath{\widehat{\mathbf{C}}^s}\xspace}
\newcommand{\estimatedcausalstructureT}{\ensuremath{\widehat{\mathbf{C}}^{s^\top}}\xspace}
\newcommand{\candidatecausalstructure}{\ensuremath{\widehat{\mathbf{B}}^s}\xspace}

\newcommand{\estimatedvariance}{\ensuremath{\widehat{\mathbf{\Sigma}}^s}\xspace}
\newcommand{\multinoise}{\ensuremath{\mathbf{Z}^s}\xspace}

\newcommand{\backbonemask}{\ensuremath{\mathbf{M}_B}\xspace}
\newcommand{\idiosyncraticmask}{\ensuremath{\mathbf{M}_I^s}\xspace}
\newcommand{\summask}{\ensuremath{\mathbf{M}^s}\xspace}
\newcommand{\weights}{\ensuremath{\mathbf{W}^s}\xspace}
\newcommand{\likelihood}{\ensuremath{\mathcal{L}}\xspace}

\definecolor{amber}{rgb}{1.0, 0.75, 0.0}
\definecolor{applegreen}{rgb}{0.55, 0.71, 0.0}

\usepackage{xcolor}
\hypersetup{
   bookmarks, pdftex,
   colorlinks=true,
   pagebackref=true, backref=page,
   linkcolor={red!50!black},
   filecolor={green!50!black},
   citecolor={green!50!black}, 
   urlcolor={blue!80!black},
}
\title[Extracting the Multiscale Causal Backbone of Brain Dynamics]{Extracting the Multiscale Causal Backbone of Brain Dynamics}

\clearauthor{
 \Name{Gabriele D'Acunto}\Email{gabriele.dacunto@uniroma1.it}\\
 \addr DIAG, Sapienza University of Rome\\
 \addr Centai Institute, Turin, Italy
 \AND
 \Name{Francesco Bonchi}\Email{francesco.bonchi@centai.eu}\\
 \addr Centai Institute, Turin, Italy
 \AND
 \Name{Gianmarco {De Francisci Morales}}\Email{gdfm@acm.org}\\
 \addr Centai Institute, Turin, Italy
 \AND
 \Name{Giovanni Petri}\Email{giovanni.petri@nulondon.ac.uk}\\
 \addr Network Science Institute, Northeastern University London\\
 \addr Centai Institute, Turin, Italy
 }

\begin{document}

\maketitle

\begin{abstract}%
  The bulk of the research effort on brain connectivity revolves around statistical associations among brain regions, which do not directly relate to the causal mechanisms governing brain dynamics.
  Here we propose the multiscale causal backbone (MCB) of brain dynamics, shared by a set of individuals across multiple temporal scales, and devise a principled methodology to extract it.

  Our approach leverages recent advances in multiscale causal structure learning and optimizes the trade-off between the model fit and its complexity.
  Empirical assessment on synthetic data shows the superiority of our methodology over a baseline based on canonical functional connectivity networks.
  When applied to resting-state fMRI data, we find sparse MCBs for both the left and right brain hemispheres.
  Thanks to its multiscale nature, our approach shows that at low-frequency bands, causal dynamics are driven by brain regions associated with high-level cognitive functions; at higher frequencies instead, nodes related to sensory processing play a crucial role.
  Finally, our analysis of individual multiscale causal structures confirms the existence of a causal fingerprint of brain connectivity, thus supporting the existing extensive research in brain connectivity fingerprinting from a causal perspective.
\end{abstract}

\smallskip
\begin{code}%
    \url{https://github.com/OfficiallyDAC/cb}%
\end{code}

\smallskip

% !TEX root =  ../main.tex
\section{Introduction}\label{sec:introduction}
The study of brain connectivity is a fundamental pursuit in neuroscience that has a rich history spanning decades~\citep{felleman1991distributed,biswal1995functional}.
The development of magnetic resonance imaging (MRI) paved the way to
\emph{connectomics}~\citep{sporns2005human,bullmore2009complex}, i.e., the study of the brain as a network, where nodes represent brain regions of interest (ROIs).\footnote{Regions of interest (ROIs), encompassing millions of neurons, are constructed by aggregating neighboring voxels, which are the finest units in 3D images.
This aggregation procedure, named parcellation, is based on the observation that adjacent voxels display correlated activities, acting as a single entity.}
Arcs among ROIs are defined either by observing anatomic fiber density (\emph{structural connectome}), or by exploiting the link between neural activity and blood flow and oxygenation\footnote{This technique is known as functional magnetic resonance imaging, or fMRI for short.} to associate a time series to each ROI and then creating a link between two ROIs that exhibit co-activation, i.e., strong correlation in their time series (\emph{functional connectome})~\citep{fox2005human}.
%,uddin2008network,smith2015positive}.

Typical measures of functional connectivity between ROIs can be categorized into undirected measures, which capture statistical dependence without specifying the direction, and directed measures, which explore statistical causation by specifying the direction of the dependencies~\citep{wiener1956theory,granger1969investigating}.
Directed connectivity measures often rely on the assumption that the cause precedes the effect, and thus focus on estimating lagged dependencies.
Examples of bivariate measures include \emph{Granger causality}~\citep{granger1969investigating} and \emph{transfer entropy}~\citep{schreiber2000measuring}.
Extensions to the multivariate case include \emph{conditional Granger causality}~\citep{ding2006granger}, \emph{partial Granger causality}~\citep{guo2008partial}, \emph{directed transfer function} (DTF, \citealp{kaminski1991new}), and \emph{partial directed coherence} (PDC, \citealp{baccala2001partial}) (see ~\citet{blinowska2011review,bastos2016tutorial} for comprehensive reviews).
%Despite venturing into the realm of causality, 
Directed functional connectivity measures cannot be considered causal in the strict sense, because they do not describe the underlying causal mechanisms that generate the observed time series.
%Thus, an improvement over the state of the art lies in using a structural causal modeling approach~\citep{pearl2009causality} for a better analysis of brain connectivity given observational data.

One step closer to causal modeling is \emph{effective connectivity}~\citep{harrison2003multivariate}, which studies directional and causal relationships between ROIs to understand how information flows within the brain.
Effective connectivity is considered a first-order data feature: a cause of the observed functional connectivity~\citep{razi2016connected}.
The reference model for effective connectivity is the \emph{dynamic causal model} (DCM) by \citet{friston2003dynamic}.
% DCM is essentially a state-space model where the behavior and coupling of $K$ hidden neuronal states, $\mathbf{x}[t] \in \real^K$, determine the variation of the observed signals $\mathbf{y}[t] \in \real^K$.
DCM is essentially a state-space model where the behavior and coupling of $K$ hidden neuronal states determine the variation of $K$ observed signals.
However, DCM is not a causal structure learning method, but rather a model for testing specific causal hypotheses.
Conversely, our work aims to infer causal relationships in a data-driven manner.

Specifically, given a group of individuals, where each individual is described by fMRI data (i.e., a set of time series corresponding to different ROIs, where the ROIs are the same across the individuals), we tackle the learning problem of discovering the \emph{multiscale causal backbone} (MCB), the largest causal network shared by individuals that still preserves personal idiosyncrasies. 
Studying the brain connectivity of groups is crucial for several tasks, such as detecting diseases and disorders, designing effective therapeutic interventions, and mapping brain networks~\citep{craddock2012whole,bassett2017network}.
Conscious that brain functions vary over different time scales~\citep{jacobs2007brain, ide2017detrended}, we approach the discovery of brain causal mechanisms from a multiscale perspective.
A \emph{time scale} refers to one of the resolutions at which a given signal (ROI time series) is analyzed.
The coarsest time scales correspond to the low-frequency components of the signal, whereas the finest to the high-frequency ones.
For instance, one might be interested in coarser time scales to examine neuronal oscillations during sleep.
Conversely, the finer time scales might be relevant to understand how neuronal populations fire in response to a specific stimulus.

We develop our method by leveraging recent advances in multiscale causal structure learning (MS-CASTLE, \citealp{d'acunto2023multiscale}) to retrieve individual multiscale DAGs, i.e., a multi-layer causal graph where the brain ROIs correspond to graph nodes, and time scales to graph layers.
%MS-CASTLE accommodates linear, instantaneous, and lagged causal dependencies that can occur at different frequency bands.
% Instantaneous causal dependencies can occur when the causal dynamics happen faster than the observation frequency, and their presence can affect the estimation of lagged interactions~\citep{hyvarinen2010estimation}.
MS-CASTLE falls under functional causal structure learning methods, as it expresses a node's value as a function of parent nodes.
It is a gradient-based causal discovery method that builds upon the continuous approximation of the acyclicity property of a graph~\citep{zheng2018dags}.
Given the learned individual multiscale DAGs, we adopt the \emph{minimum description length} (MDL) principle to detect the MCB which optimizes the trade-off between the model's fit (i.e., the likelihood of observing the data given the MCB and individual idiosyncrasies) and its complexity (i.e., the number of arcs in the multiscale causal graph consisting of the MCB and the individual idiosyncrasies).

\noindent
\textcolor{black}{
\spara{Contributions.}
From a methodological perspective:
\begin{squishlist}
    \item We formally define the learning problem of discovering the MCB from a sample of individual fMRI recordings in a \emph{resting state}.
    \item We devise a methodology to solve this learning task that first uses MS-CASTLE~\citep{d'acunto2023multiscale} to learn individual multiscale DAGs, and then runs a score-based search procedure on top of the learned individual multiscale DAGs.
    The proposed score-based search procedure is related to score-based causal structure learning methods~\citep{heckerman1995learning,chickering2002optimal,gu2020learning}.
    The score function we optimize however represents a generalization to the multi-subject, multiscale setting of the widely used scores for learning Gaussian Bayesian networks, e.g., the BIC score~\citep{schwarz1978estimating}.
    Our methodology differs from multi-domain causal structure learning methods~\citep{ghassami2018multi,zeng2021causal,perry2022causal}, which instead typically learn a single causal structure from multi-domain data.
    Indeed, the weights of the arcs composing the MCB vary over the subjects. 
    This way, we can learn the causal structure shared across multiple subjects while preserving the idiosyncracies of each individual.
    Our empirical assessment on synthetic data shows the superiority of our method over a baseline based on functional connectivity networks.
\end{squishlist}
}

\textcolor{black}{
From an application standpoint:
\begin{squishlist}
    \item Using the proposed methodology, we study a dataset of resting-state functional magnetic resonance imaging (rs-fMRI), constituted by $100$ healthy subjects, publicly available as part of the Human Connectome Project (HCP,~\citealp{smith2013resting}).
    We find that MCBs can be detected for both the left and right hemispheres and that they are sparse.
    Our method also allows us to discover the ROIs driving the brain activity at different frequency bands. In particular, we find two sets of ROIs that play different roles depending on the considered frequency band.  
    The first set of nodes, relevant at lower frequencies, is associated with high-level cognitive functions, such as self-awareness and introspection; 
    while the second set, significant for higher frequencies, includes nodes that are key drivers in sensory processing.
    \item We compare the learned MCBs with the single-scale causal backbones (SCBs).
    Despite the agreement of single-scale and multiscale analyses on the identification of the key drivers within the system at hand, the SCBs result in an aggregated version of the MCBs.
    This aggregation hinders the capability to localize (in the frequency domain) changes in the causal structure, and to distinguish that higher-level cognitive functions prevail at lower frequencies, i.e., below $0.1$ Hz, as also supported by \citet{biswal1995functional,fox2005human,van2010intrinsic}.
    \item We investigate the viability of \emph{causal fingerprinting}, i.e., using the variability of individual multiscale DAGs to identify individuals in large-scale studies.
    We show that causal fingerprinting is indeed possible at various hard-thresholding levels and that it outperforms canonical fingerprinting based on functional connectivity alone.
\end{squishlist}
}

\noindent
\textcolor{black}{
\spara{Roadmap.}
\Cref{sec:preliminaries} provides the notation and background information needed to formally define the problem we tackle in this paper.
Next, \Cref{sec:learningMDAGs} details the learning problem of individual multiscale DAGs and its solution.
Then, \Cref{sec:learningMCB} formally defines the MCB, poses the corresponding learning problem, and devises the score-based search procedure to solve it.
\Cref{sec:synth_data} compares our proposal to baseline methods on synthetic data, while \Cref{sec:results} shows our application to fMRI resting-state data.
Finally, \Cref{sec:conclusions} concludes and outlines future research directions.
}
% !TEX root =  ../main.tex
\section{Preliminaries}\label{sec:preliminaries}
% Here we provide the notation and background information needed to formally define the problem we tackle in this paper.

%To properly pose our learning problem formulation, let us introduce the notation we use henceforth, along with essential notions about \emph{(i)} how we analyze input fMRI data at different time resolutions, and \emph{(ii)} how we adhere to the MDL.  

\spara{Notation.} 
The range of integers from $1$ to $Y$ is denoted by $[Y]$.
% , $[Y]_0$ if the zero is included.
Scalars are lowercase, $y$, vectors are lowercase bold, $\mathbf{y}$, and matrices are uppercase bold, $\mathbf{Y}$.
Given $\mathbf{y} \in \real^{Y}$, the diagonal matrix of size $Y$ by $Y$ having $\mathbf{y}$ as main diagonal is $\mathrm{diag}(\mathbf{y})$.
The set of block-diagonal matrices of $\real^{\bar{Y} \times \bar{Y}}, \, \bar{Y}=BY \in \nat$, having $B$ blocks of size $Y$ by $Y$ is denoted by $\multiblock_{\bar{Y}}$.
% Analogously, $\multilogicblock_{\bar{Y}}$ denotes the set of block-diagonal matrices of $\{0,1\}^{\bar{Y} \times \bar{Y}}, \, \bar{Y}=BY \in \nat$, having $B$ blocks of size $Y$ by $Y$.
The Frobenious norm of a matrix is denoted by $\left|\left| \mathbf{Y} \right|\right|_{\mathrm{F}}$, while the Hadamard product is $\circ$. 
$N(\boldsymbol{\mu}, \boldsymbol{\Sigma})$ is the multivariate normal distribution with mean $\boldsymbol{\mu} \in \real^{K}$ and covariance matrix $\boldsymbol{\Sigma} \in \real^{K\times K}$.

\spara{Signal multiscale representation via wavelet transform.}
%Similarly to \citet{d'acunto2023multiscale}, 
To analyze input data at different time resolutions we use the wavelet transform (see \citealp{percival2000wavelet} for detailed discussion).
Let us consider $S \in \nat$ individuals.
For each individual $s \in [S]$, we are given as input the fMRI data set $\data=[\signal_1, \ldots, \signal_K]^\top \in \real^{K \times N}$.
Here $\signal_i \in \real^{N}$ represents the signal corresponding to the $i$-th ROI, $K$ is the number of ROIs, and $N=M\cdot2^J,\, M \text{ and }J \in \nat$, is the length of the signals.
Furthermore, we leverage the following assumption, which is common when dealing with resting-state fMRI data~\citep{hlinka2011functional,fiecas2013quantifying,garg2013gaussian}.

\begin{assumption}\label{a1}
    The input data set $\data$ consists of i.i.d. samples $\sample \sim N(\mathbf{0}, \mathbf{\Omega}^s)$, for each $s \in [S]$.
\end{assumption}

Each column $\sample$ within the data set represents a recorded sample collected at a specific time point denoted as $n \cdot \Delta t$, where $\Delta t$ is the sampling interval.
The wavelet decomposition at level $J-1$ transforms each individual time series $\mathbf{x}_i^s$ into $J-1$ vectors of wavelet coefficients, in addition to an extra vector of scaling coefficients. 
To delve deeper into this, the vector of wavelet coefficients indexed by $j$ characterizes the fluctuations within $\mathbf{x}^s_i$ at a temporal scale of $2^{j-1} \cdot \Delta t$. 
Essentially, this vector represents the frequency band $[1/2^{j+1},1/2^j]$. 
These vectors of wavelet coefficients capture changes within the input signal across temporal scales spanning from $\Delta t$ to $2^{J-2} \cdot \Delta t$, which translates to frequencies encompassed between $1/2^J$ and $1/2$.
Conversely, the vector of scaling coefficients encapsulates information regarding variations occurring at the scale $2^{J-1}$ and coarser, related to frequencies slower than $1/2^J$.
We denote the finest scale with $j=1$ and the coarsest one with $j=J$.

Thus, by means of stationary wavelet transform (SWT,~\citealp{nason1995stationary}), the input fMRI data set $\data$ is converted to a multiscale data set $\multidata \in \real^{\bar{K} \times N}, \, \bar{K}=JK \in \nat$, where at each sample \sample corresponds its multiscale representation 
\begin{equation}\label{eq:multiscalerepresentation}
    \multisample=[\tilde{x}_1^{s,1}[n], \tilde{x}_2^{s,1}[n], \ldots, \tilde{x}_K^{s,1}[n], \ldots, \tilde{x}_1^{s,J}[n], \tilde{x}_2^{s,J}[n], \ldots, \tilde{x}_K^{s,J}[n]]^\top \in \real^{\bar{K}}.
\end{equation}
Here, $\tilde{x}_i^{s,j}[n]$ represents the wavelet coefficient at scale $j$, of the $i$-th ROI, for the individual $s$, at time $n$.
As shown in \citet{percival2000wavelet}, the wavelet transform is a linear transformation that enables the decomposition of an input signal without loss of information.
Consequently, the input signal can be perfectly reconstructed from the obtained wavelet coefficients. 
Given Assumption \ref{a1}, the following lemma holds.
\begin{restatable}{mylemma}{waveletgauss}\label{lem:waveletgauss}
    The wavelet coefficient $\tilde{x}_i^{s,j}[n]$ is distributed according to a zero-mean Gaussian distribution, for all  $s \in [S], \, j\in [J], \text{ and } n \in [N]$.
\end{restatable}

\begin{proof}\renewcommand{\jmlrQED}{}
    See Appendix \ref{app:lemma1}.
\end{proof}

\spara{Model selection via MDL.}
Our approach differs from the existing work in two ways.
First, we leverage multiscale DAGs instead of functional connectivity networks at the individual level.
Second, we do not apply averaging over the individuals' connectivity networks, rather we learn the MCB in a data-driven manner, abiding by the MDL principle.
Since we can assume to conduct the MCB discovery in a Gaussian setting (as formally stated in \Cref{sec:learningMCB}), leveraging the well-known equivalence results for Gaussian Bayesian networks by \citet{chickering2013transformational}, we can adhere to MDL principle by using the BIC score for model selection.
Specifically, consider a data set $\mathbf{D} \in R^{Q\times R}$, and a causal structure $G=(\mathcal{V}_G, \mathcal{E}_G)$ parameterized by $\boldsymbol{\theta}_{G}$, and entailing the joint distribution of the data.
By denoting the \emph{likelihood} with $\likelihood(\boldsymbol{\theta}_G \mid D)$, the maximum likelihood reads as 
\begin{equation}\label{eq:ML}
    \mathrm{ML}(D) = \max_{\boldsymbol{\theta}_G} \likelihood(\boldsymbol{\theta}_G \mid D).
\end{equation}
Hence, by using \eqref{eq:ML}, we can write the BIC score~\citep{schwarz1978estimating} as 
\begin{equation}\label{eq:BIC}
    \mathrm{BIC}(G,D) = - 2\log{\mathrm{ML}(D)} + \xi |\mathcal{E}_G|;
\end{equation}
where $\xi=\log{R}$.

% !TEX root =  ../main.tex
\section{Learning individuals multiscale DAGs}\label{sec:learningMDAGs}
Our goal requires formulating two learning problems, where the output of the first feeds as input the second.
The first problem concerns learning the multiscale DAGs of the individuals. The second, learning the MCB from these causal graphs.
Let us start with the former.

% Given $\multidata$, i.e.,  the multiscale representation of the input fMRI data for the $s$-th individual introduced in \Cref{sec:preliminaries}, we use the weighted adjacency of the multiscale linear DAG learned from $\multidata$ as the multiscale causal connectivity matrix for the $s$-th individual.

\begin{definition}[Multiscale linear DAG for the $s$-th individual]\label{def:imDAG}
    The multiscale linear DAG for the individual $s$ is a multi-layer DAG $G^s=(\vertexset, \causalset)$ with $J$ independent layers. Each layer $j \in [J]$ corresponds to the $j$-th time scale, $\vertexset=[K]$, and $\causalset=\{\mathcal{E}^{s,1} \cup \mathcal{E}^{s,2} \cup \ldots \cup \mathcal{E}^{s,J}\}$. At layer $j$, the node $l$, corresponds to the representation of the signal $\signal_l$ at time scale $j$, namely $\tilde{\mathbf{x}}_l^{s,j}$. The directed arc $e_{lm}^{s,j} \in \causalsetj$ from node $l$ to $m$ has weight $c_{lm}^{s,j} \in \real$, representing the strength of a linear causal connection from $l$ to $m$ occurring at time scale $j$. Inter-layer arcs are forbidden. 
\end{definition}

By following \citet{d'acunto2023multiscale}, given $\multidata$, i.e.,  the multiscale representation of the input fMRI data for the $s$-th individual introduced in \Cref{sec:preliminaries}, the multiscale linear DAG in Definition \ref{def:imDAG} underlies the following linear causal model: 
\begin{equation}\label{eq:multiSEM}
    \multidata = \causalstructureT \multidata + \multinoise.
\end{equation}
In \Cref{eq:multiSEM}, the weighted adjacency $\causalstructure \in \multiblock_{\bar{K}}$ is nilpotent, and we assume that $\multinoise \in \real^{\bar{K}\times N}$ is an i.i.d. multivariate Gaussian noise, $\multinoise \sim N(\mathbf{0}, \mathbf{\Sigma}^s), \, \mathbf{\Sigma}^s=\mathrm{diag}(\boldsymbol{\sigma}^{s^2}), \, \boldsymbol{\sigma}^{s^2}=[(\sigma^{s}_1)^2,\ldots,(\sigma^s_{\bar{K}})^2]$.
Notice that the block-diagonal structure of $\causalstructure$ complies with the independence among layers.

\begin{problem}
    For each individual $s \in [S]$, the matrix $\causalstructure$ in \Cref{eq:multiSEM} is learned by solving the following continuous optimization problem:
    \begin{equation}
    \begin{aligned}
        \min_{\estimatedcausalstructure \in \multiblock} &\quad  \frac{1}{2}\left|\left|\multidata-\estimatedcausalstructureT \multidata\right|\right|_{\mathrm{F}}^2 + \lambda \left|\left|\estimatedcausalstructure\right|\right|_{1}\, ,\\
       \mathrm{subject \,to}&  \quad  h(\estimatedcausalstructure)=\mathrm{Tr}\left(e^{\estimatedcausalstructure \circ \estimatedcausalstructure}\right)-\bar{K} = 0;
    \end{aligned}
    \tag{P1}\label{eq:MSCASTLE}
\end{equation}
    where $\lambda \in \real$ is a tunable parameter that promotes sparse solutions.
\end{problem}

Problem~\eqref{eq:MSCASTLE} is non-convex due to the acyclicity constraint introduced by \citet{zheng2018dags}.
To solve it, we resort to the gradient-based linear method MS-CASTLE~\citep{d'acunto2023multiscale}.
MS-CASTLE solves Problem~\eqref{eq:MSCASTLE} by introducing a linearization of the acyclicity constraint, and then by leveraging the alternating direction method of multipliers (ADMM, \citealp{boyd2011distributed}).
MS-CASTLE has been empirically proven to be robust to \emph{(i)} different choices of the orthogonal wavelet family employed to obtain the multiscale representation in \eqref{eq:multiscalerepresentation}, and \emph{(ii)} different distributions of the noise $\multinoise$ in \eqref{eq:multiSEM}.
The multiscale weighted adjacencies $\{\estimatedcausalstructure\}, \, s \in [S]$, induce the multiscale linear DAGs $\{\estimatedindividualcausalgraph\}$, with $\estimatedindividualcausalgraph=(\vertexset, \estimatedindividualcausalset)$.

% !TEX root =  ../main.tex
\section{Learning the multiscale causal backbone}\label{sec:learningMCB}
We can now move to our second challenge: learning the MCB.
% Starting from $\{\estimatedcausalstructure\}$, let us consider the corresponding unweighted adjacency $\{\estimatedadjacency \in \{0,1\}^{\bar{K \times \bar{K}}}\}$, such that $\hat{b}_{lm}^{s,j}=1$ iff $\hat{c}_{lm}^{s,j}\neq0$, and $0$ otherwise.
To formally define the MCB, we need to introduce the concepts of \emph{(i)} $p$-persistent, directed, and unweighted arc $u_{lm}^j$; \emph{(ii)} candidate universe $\candidateuniverse$; and \emph{(iii)} set of idiosyncratic causal arcs $\idiosyncraticset$.

\begin{definition}[$p$-persistent arc]\label{def:ppersistentArc}
    The directed, unweighted arc $u_{lm}^j$ is $p$-persistent iff $\widehat{\mathbf{C}}^{s} \ni \hat{c}_{lm}^{s,j}\neq0 \\\; \forall s \in \mathcal{P}$, such that $\mathcal{P} \subseteq [S]$ and $\mid \mathcal{P}\mid>p$.    
\end{definition}
% \mynote[from=gdfm]{This definition is ambiguous. I think it should read as follows instead.
% \begin{definition}[$p$-persistent arc]\label{def:ppersistentArc}
%     The directed, and unweighted arc $u_{lm}^j$ is $p$-persistent iff $\lvert \{ P \mid \forall s \in P \subseteq [S], \hat{c}_{lm}^{s,j} \in \widehat{\mathbf{C}}^{s} \neq 0 \} \rvert \geq p$.
% \end{definition}
% Also, wouldn't it be more elegant if $p$ was a fraction w.r.t. S?
% }
Informally, we say an arc is $p$-persistent if it appears (i.e., its weight in the $j$-th layer is non-zero) in more than $p$ of the learned weighted adjacency matrices.

\begin{definition}[Candidate universe]\label{def:candidateU}
    Given an integer $p$, the candidate universe $\candidateuniverse$ is defined as
    \begin{equation}\label{eq:candidateU}
        \candidateuniverse \coloneqq \{u_{lm}^j \, \mid \, u_{lm}^j \text{ is $p$-persistent}, j \in [J], \text{ and } l,m \in \vertexset\}.
    \end{equation}
\end{definition}
That is, the candidate universe comprises the set of all $p$-persistent arcs.

\begin{definition}[Set of idiosyncratic causal arcs]\label{def:idiosyncraticP}
    Given the multiscale linear DAG $\estimatedindividualcausalgraph=(\vertexset, \estimatedindividualcausalset)$ for the $s$-th individual, the set of idiosyncratic causal arcs is 
    \begin{equation}\label{eq:idiosyncraticP}
        \idiosyncraticset \coloneqq \estimatedindividualcausalset \setminus \candidateuniverse.
    \end{equation}
\end{definition}
% \mynote[from=gdfm]{Is this true? What happens to the $p$-persistent arcs that are not selected to be in the backbone?}

The sets $\{\idiosyncraticset\}$ are instrumental in isolating the information within the data that we want to explain only by the arcs in $\candidateuniverse$, adhering to MDL.
We are now ready to formally define the MCB.

% \begin{definition}[Multiscale causal backbone, MCB]\label{def:mcb}
%     The multiscale causal backbone shared by $S$ individuals is an unweighted multiscale linear DAG $A_p=(\vertexset, \backboneset)$, where $\backboneset \subseteq \candidateuniverse$, and the presence of an arc $a_{lm}^j \in \backboneset$ complies with the MDL principle. 
% \end{definition}
% \mynote[from=gdfm]{I would put the desiderata and constraints that the MCB has in the definition. Otherwise it feels quite hollow.}
\begin{definition}[Multiscale causal backbone, MCB]\label{def:mcb}
Given a collection of individual multiscale causal DAGs $\{\estimatedindividualcausalgraph\}$ and a subset of $p$-persistent arcs $\candidateuniverse \subseteq \{\estimatedcausalset^1 \cup \estimatedcausalset^2 \cup \ldots \cup \estimatedcausalset^S\}$, the multiscale causal backbone MCB is the unweighted, multiscale, linear DAG $(\vertexset, \backboneset), \, \backboneset \subseteq \candidateuniverse$, that minimizes the description length of the collection of individual multiscale DAGs $\{\estimatedindividualcausalgraph\}$ when represented as the union of the MCB and individual (i.e., p-persistent \& idiosyncratic) arcs.

%The multiscale causal backbone shared by $S$ individuals is the most concise unweighted multiscale linear DAG $A_p=(\vertexset, \backboneset)$ adhering to the MDL principle, where $\backboneset \subseteq \candidateuniverse$. It explains the information in the data that cannot be accounted for by individual multiscale DAGs containing non-$p$-persistent idiosyncratic arcs.
\end{definition}

% According to Definition \eqref{def:mcb}, the MCB represents a shared causal structure that might contain also causal arcs that are not present in all the estimated multiscale linear DAGs $\{\estimatedindividualcausalgraph\}$.
% This choice is conscious, since only considering the case $\mathcal{P}=[S]$ in Definition \eqref{def:ppersistentArc} is too stringent and ignores estimation errors in $\{\estimatedindividualcausalgraph\}$, which might result in the exclusion of relevant arcs.

% According to Definition \eqref{def:mcb}, the MCB represents the most concise causal structure shared by a sample of individuals, explaining the information in the data that cannot be accounted for by individual multiscale DAGs containing non-$p$-persistent idiosyncratic arcs, as defined in Definition \eqref{def:idiosyncraticP}. 
According to Definition~\eqref{def:mcb}, the MCB strictly relates to the value of $p$.
Specifically, different choices for $p$ are tied to different assumptions in the generative model.
Choosing $p=0$ implies assuming that the fMRI data for the individuals are generated by a single common multiscale DAG. 
This is a very strong assumption, experimentally opposed by the evidence of causal fingerprinting provided in \Cref{sec:results}.
Setting $p=S-1$ means assuming that the individual multiscale DAGs are composed of the union of the idiosyncratic arcs $\{\idiosyncraticset\}$ and a set of shared backbone arcs \backboneset. 
Moreover, it potentially considers idiosyncratic arcs that occur in $S-1$ individual multiscale DAGs. 
In reality, we expect to have idiosyncratic arcs that occur in fewer individual DAGs. 
Finally, it assumes that MS-CASTLE has maximum recall in the learning task for each subject $s$. 
Clearly, this hypothesis is also unrealistic given the difficulty of the causal structure learning task from real-world data.
A third, more reasonable choice is $p \in (0,S-1)$. 
In this case, the assumption of perfect discovery by MS-CASTLE is relaxed, which allows the MCB to contain causal arcs that are not present in all DAGs $\{\estimatedindividualcausalgraph\}$. 
In addition, idiosyncrasy can be defined at a level $p<S-1$.
Of course, here the value of $p$ represents a trade-off between the dissimilarity of the idiosyncratic components along the subjects and the potential size of the MCB.
Given that $p$ is a hyperparameter of the algorithm, tuning it automatically from data would be desirable.
However, as shown in \Cref{app:cost}, the computational cost of the score-based search procedure we employ is quite high for low values of $p$.
We plan to develop a data-driven procedure for fine-tuning $p$ in future work.

\begin{problem}
    Given \Cref{lem:waveletgauss} and the Gaussianity assumption of $\multinoise$ in \Cref{eq:multiSEM}, considering as input data $\{\multidata\}$, the MCB in Definition~\eqref{def:mcb} can be learned by solving the following problem:
    \begin{equation}
    \min_{\estimatedbackboneset \, \subseteq \, \candidateuniverse} \mathrm{BIC}(\{\candidatecausalgraph\}, \{\multidata\}).
    \tag{P2}\label{eq:learningproblem}
    \end{equation}
    Here $\candidatecausalgraph=(\vertexset, \estimatedbackboneset \cup \idiosyncraticset)$, $s\in [S]$, is the multiscale causal structure for individual $s$, whose arc set consists of the union of the arcs in the learned MCB and those in the idiosyncratic set $\idiosyncraticset$ in Definition~\eqref{def:idiosyncraticP}.
\end{problem}

% To transform Definition~\ref{def:mcb} into a proper learning problem for discovering the MCB that abides by the MDL principle, we use the BIC score in \Cref{eq:BIC}. 
% Specifically, our choice is justified by \Cref{lem:waveletgauss} and the Gaussianity assumption of $\multinoise$ in \Cref{eq:multiSEM}. 
In Problem~\eqref{eq:learningproblem} the BIC score operationalizes the MDL, given the Gaussianity of our setting (cf. \Cref{sec:preliminaries}).
Interestingly, in Problem~\eqref{eq:learningproblem}, the weights of the arcs in $\candidatecausalgraph$ vary over $s \in [S]$. % which also belong to the learned MCB $\widehat{A}_p$
This property reflects the spirit of our work, which is to learn an MCB while preserving features at the individual level.

To solve Problem~\eqref{eq:learningproblem}, we propose a score-based causal structure learning method.
Let us indicate with $\candidatecausalstructure \in \multiblock_{\bar{K}}$ the ML estimate for causal coefficients corresponding to $\candidatecausalgraph$, where the $K$ by $K$ block associated with the $j$-th scale is $\widehat{\mathbf{B}}^{s,j}=[\hat{\mathbf{b}}_1^{s,j}, \ldots, \hat{\mathbf{b}}_K^{s,j}]$.
Here, $\hat{\mathbf{b}}_k^{s,j}$ are the weights of the incoming causal arcs for the node $k$, at scale $j$ and for individual $s$.

Hence, given \emph{(i)} the Gaussianity of the setting, \emph{(ii)} the independence among the $s$ individuals, and \emph{(iii)} the independence of the time scales in the $\{G^s\}$, the log-likelihood evaluated at the ML estimates $(\candidatecausalstructure,\estimatedvariance)$, $s \in [S]$,  reads as
\begin{equation}\label{eq:loglike}
    \ell(\{\candidatecausalstructure\},\! \{\estimatedvariance\} \mid \{\multidata\})\! = \!\sum_s^S \! \sum_j^J \! \sum_k^K \! \left[ -\frac{N}{2} \!\log{\left( \hat{\sigma}_{k}^{s,j}\right)^2} \!-\! \frac{1}{\left( \hat{\sigma}_{k}^{s,j}\right)^2} \! \left|\left| \multisignal \!-\! \hat{\mathbf{b}}_k^{s,j} \widetilde{\mathbf{X}}^{s,j} \right|\right|_2^2 \right].
\end{equation}

Then, given \Cref{eq:BIC,eq:loglike}, the score to optimize for our case is
\begin{equation}\label{eq:score}
    \mathrm{score}(\{\candidatecausalgraph\},\{\multidata\}) = -2 \log \ell(\{\candidatecausalstructure\},\! \{\estimatedvariance\} \mid \{\multidata\}) + \xi \sum_{s\in [S]} |\estimatedbackboneset \cup \idiosyncraticset|.
\end{equation}

\begin{remark}
    If $\xi=\log{N}$ the score in \Cref{eq:score} corresponds to the BIC. If $\xi= 2 \log{K}$, the score corresponds to the risk inflation criterion (RIC,~\citealp{foster1994risk}).
    RIC is a penalized version of the BIC for multiple regression, suitable when the number of nodes is large, i.e., $K>\sqrt{N}$.
    Since in our setting  $K>\sqrt{N}$ (cf., \Cref{sec:results}), we use the RIC score.
\end{remark}

\textcolor{black}{
The score in \Cref{eq:score} is \emph{consistent} (cf. Definition 18.1 and Theorem 18.2 in \citealp{koller2009probabilistic}). 
Therefore, as the amount of data increases, the true underlying structure minimizes the score. 
}
Given the form of $\ell(\{\candidatecausalstructure\},\! \{\estimatedvariance\} \mid \{\multidata\})$ in \Cref{eq:loglike}, according to Proposition 18.2 in \citet{koller2009probabilistic}, the score in \Cref{eq:score} is \emph{decomposable} since it can be written as the sum over $s$ and $j$ of family scores corresponding to each node $k \in \vertexset$, mathematically,
\begin{equation}\label{eq:familyscore}
    \mathrm{RIC}_k(\{\candidatecausalgraph\},\{\widetilde{\mathbf{X}}^{s,j}\}) = -\frac{N}{2} \!\log{\left( \hat{\sigma}_{k}^{s,j}\right)^2} \!-\! \frac{1}{\left( \hat{\sigma}_{k}^{s,j}\right)^2} \! \left|\left| \multisignal \!-\! \hat{\mathbf{b}}_k^{s,j} \widetilde{\mathbf{X}}^{s,j} \right|\right|_2^2 + \xi |\widehat{\mathcal{A}}_{p,k}^j \cup \mathcal{I}_{p,k}^{s,j}|.
\end{equation}
Here, $\widehat{\mathcal{A}}_{p,k}^j \subset \estimatedbackboneset$ and $\mathcal{I}_{p,k}^{s,j} \subset \idiosyncraticset$ are the subsets of incoming arcs in $k$ at scale $j$.
Intuitively, $\mathrm{RIC}_k(\{\candidatecausalgraph\},\{\widetilde{\mathbf{X}}^{s,j}\})$ measures how well a set of nodes perform as parents for the child node $k$, for subject $s$ and at scale $j$, while penalizing larger parents sets. 

\textcolor{black}{
Together, the consistency and decomposability properties imply \emph{local consistency} of the score (cf. Definition 6 and Lemma 7 in \citealp{chickering2002optimal}).
This property in turn guarantees asymptotically that an arc absent from the MCB does not minimize the score in \Cref{eq:score}. 
}

\textcolor{black}{Given that our method is iterative, let us denote with $\widehat{\mathcal{A}}_{p,t}=\{\widehat{\mathcal{A}}_{p,t}^1 \cup \widehat{\mathcal{A}}_{p,t}^2\cup \ldots \cup \widehat{\mathcal{A}}_{p,t}^J\}$ the solution at iteration $t$, with $\widehat{\mathcal{A}}_{p,t}^j$ being the set of arcs for the scale $j$.
Let us consider that the last added $p$-persistent arc for the scale $j$ is $q \rightarrow r$, namely $u_{qr}^j$.
Let us indicate with $\candidateuniverse^j\subseteq \candidateuniverse$ the subset of $p$-persistent arcs at scale $j$.
The decomposability of the score in \Cref{eq:score} implies that at scale $j$, before evaluating the addition of another $p$-persistent arc from $\candidateuniverse^j$, we need to update only the family score for the child node $r$ given in \Cref{eq:familyscore}, i.e., $\mathrm{RIC}_r(\{\candidatecausalgraph\},\{\widetilde{\mathbf{X}}^{s,j}\})$, for each $s \in [S]$.
Accordingly, we only consider the $p$-persistent arcs $u_{l r}^j \in \candidateuniverse^j$ ending in $r$, where $\candidateuniverse^j$ no longer contains previously evaluated arcs.
}

\textcolor{black}{Specifically, for each of these $p$-persistent arcs, we compare the models
$$ \mathrm{M}_0: \, \widehat{\mathcal{A}}^j_{p,t+1}=\widehat{\mathcal{A}}^j_{p,t}, \quad \text{and} \quad \mathrm{M}_1: \, \widehat{\mathcal{A}}^j_{p,t+1}=\widehat{\mathcal{A}}^j_{p,t} \cup \{u_{l r}^j\},$$
by means of the following linear regression model
\begin{equation}\label{eq:regression}
    \multisignal_r = b_{l r}^{s,j} \multisignal_{l} + \sum_{y \in \mathcal{P}^{s,j}_r} b_{yr}^{s,j} \multisignal_{y} + \mathbf{z}^{s,j}_r, \quad \forall s \in [S],
\end{equation}
where $\mathcal{P}^{s,j}_r$ is the current parent set of the node $r$, at scale $j$ for individual $s$.
Hence, we compute the least-squares estimates for the regression coefficients, for all the $S$ individuals in a parallel fashion.
$\mathrm{RIC}^{\mathrm{M}_0}$ corresponds to the case in which the log-likelihoods of all the individuals are evaluated at the least-squares estimates with $b_{l r}^{s,j}=0$, whereas $\mathrm{RIC}^{\mathrm{M}_1}$ to the case in which $b_{l r}^{s,j}\neq0$.
At this point, each $p$-persistent arc $u_{l r}^j$ ending in $r$ is associated with an updated $\Delta_{\mathrm{RIC}}^j=\mathrm{RIC}^{\mathrm{M}_1}-\mathrm{RIC}^{\mathrm{M}_0}=\sum_{s\in [S]} \Delta_{\mathrm{RIC}}^{s,j}$.
Thus, $\Delta_{\mathrm{RIC}}^j$ is the sum of contributions coming from all the $S$ individuals, each possibly having distinct $\hat{\mathbf{b}}_r^{s,j}$.
}

\textcolor{black}{Subsequently, we add to $\estimatedbackboneset^j$ the $p$-persistent arc $\candidateedge^j \in \candidateuniverse^j$ associated with the lowest $\Delta_{\mathrm{RIC}}^j<0$, while not inducing cycles in the solution $\widehat{A}_p$.
The procedure ends when none of the remaining $p$-persistent arcs reduces the RIC if added to $\estimatedbackboneset^j$.
Due to independence, the search procedure is parallelizable over the time scales $j \in [J]$ and the individuals $s \in [S]$, with further benefits to the computational cost. 
The overall procedure is given in \Cref{alg:search}.
}

\begin{algorithm}[ht!]
\caption{MCB score-based search procedure}\label{alg:search}
\begin{algorithmic}
\footnotesize
\State {\bfseries Input: $\candidateuniverse$, $\idiosyncraticset$} 
\State {\bfseries Output: $\estimatedbackboneset$} 
\State {\bfseries Initialize: $\estimatedbackboneset \gets \emptyset$}
\State {\bfseries do in parallel over $j \in [J]$}
    \State {\quad \bfseries do in parallel over $s \in [S]$}
        \State {\quad \quad \bfseries do in parallel over $k \in \vertexset$}
            \State {\quad \quad \quad Compute $\mathrm{RIC}_k(\{\candidatecausalgraph\},\{\widetilde{\mathbf{X}}^{s,j}\})$} in \Cref{eq:familyscore} \Comment{At the beginning $\candidatecausalgraph=(\vertexset, \idiosyncraticset)$}
% \State {\bfseries do in parallel over $j \in [J]$}
    \State {\quad \bfseries while $\exists \, \candidateedge^j \in \candidateuniverse^j \, \mid \, \Delta_{\mathrm{RIC}}^j<0$ do:}
        \State {\quad \quad Select $\candidateedge^j$ for which $\Delta_{\mathrm{RIC}}^j$ is mimimum}
        \State{\quad \quad \textbf{if } $\candidateedge^j$ does not induce cycles \textbf{then:}}
            \State {\quad \quad \quad $\estimatedbackboneset^j=\estimatedbackboneset^j \cup \{\candidateedge^j\}$} \Comment{Add the arc to the MCB}
            \State {\quad \quad \quad \bfseries do in parallel over $s \in [S]$}
            \State {\quad \quad \quad \quad Update $\mathrm{RIC}_m(\{\candidatecausalgraph\},\{\widetilde{\mathbf{X}}^{s,j}\})$ in \Cref{eq:familyscore} corresponding to the child node $m$}
        \State{\quad \quad \bfseries end if}
        \State {\quad \quad $\candidateuniverse^j=\candidateuniverse^j \setminus \{\candidateedge^j\}$} \Comment{Remove the evaluated arc}
    \State {\quad \bfseries end while}
\State{$\estimatedbackboneset=\bigcup\limits_{j \in [J]} \estimatedbackboneset^j$}
\end{algorithmic}
\end{algorithm}
\vspace{-\baselineskip}

% !TEX root =  ../main.tex
\section{Empirical assessment on synthetic data}\label{sec:synth_data}

In this section, we present the empirical assessment of the proposed method on synthetic data.
In this way, we can effectively evaluate the performance of the proposed methodology since we exactly know the MCB, along with the set of idiosyncratic arcs of individuals, generating the observed data.
In detail, we use the following data-generating process.

\spara{Data generation.}
We generate $50$ data sets, each consisting of $S=100$ individuals.
Each individual is described by $K=10$ ROIs, each with a time series of length $N=1200$ determined by an underlying causal structure that is the union of a shared causal backbone and the individual idiosyncratic set.
% We consider $(S,K,N)=(10,100,1200)$, and w.l.o.g., $J=1$.
Furthermore, w.l.o.g., we consider the case $J=1$.
Indeed, since the time scales are independent, the performance of \Cref{alg:search} is not affected by the number of time scales.
We generate a strictly lower triangular masking matrix $\backbonemask \in \{0,1\}^{K \times K}$, where $[\backbonemask]_{ij} \sim B(0.25)$ and $B$ is the Bernoulli distribution. 
The non-zero entries of this matrix correspond to the arcs in the causal backbone.
Similarly, we generate strictly lower triangular masking matrices $\idiosyncraticmask \in \{0,1\}^{K \times K}$, where $[\idiosyncraticmask]_{ij} \sim B(0.5)$, with $s \in [S]$.
The non-zero entries of these matrices, which are not in the backbone, represent the idiosyncratic connection for the individuals.
Hence, $\summask = \backbonemask + \idiosyncraticmask$ is the adiaciency matrix corresponding to $G^{s}$ for individual $s$.
At this point, we sample $\weights \in \real^{K \times K}$ from a uniform $U(-1,1)$, thus obtaining the causal matrix for each individual $\causalstructure = \summask \circ \weights, \, \forall s \in [S]$, representing the weights of the arcs of $G^{s}$.
Finally, we generate data according to \Cref{eq:multiSEM}.
As mentioned above, we build $50$ data sets according to this procedure.
% Specifically, we generate $50$ data sets, each consisting of $S=100$ individuals.
% Each individual is described by $K=10$ ROIs, each with a time series of length $N=1200$ determined by an underlying causal structure that is the union of a shared causal backbone and the individual idiosyncratic set.
% \Cref{subapp:addresults1} describes the generating process more in detail.

\spara{Results.} 
We compare the proposed methodology to baselines based on widely used measures of functional connectivity, in terms of F1 score and structural Hamming similarity (SHS).
\Cref{app:metrics} describes the metrics used.
In particular, among the connectivity measures outlined in \Cref{sec:introduction}, we consider \emph{(i)} directed transfer function (DTF) and partial directed coherence (PDC) among the directed measures, and \emph{(ii)} Pearson's correlation, partial correlation, and mutual information among the undirected ones.
\Cref{subapp:addresults1} provides details on the application of the baselines.

\begin{figure}[t]
    \centering
    \includegraphics[width=\textwidth]{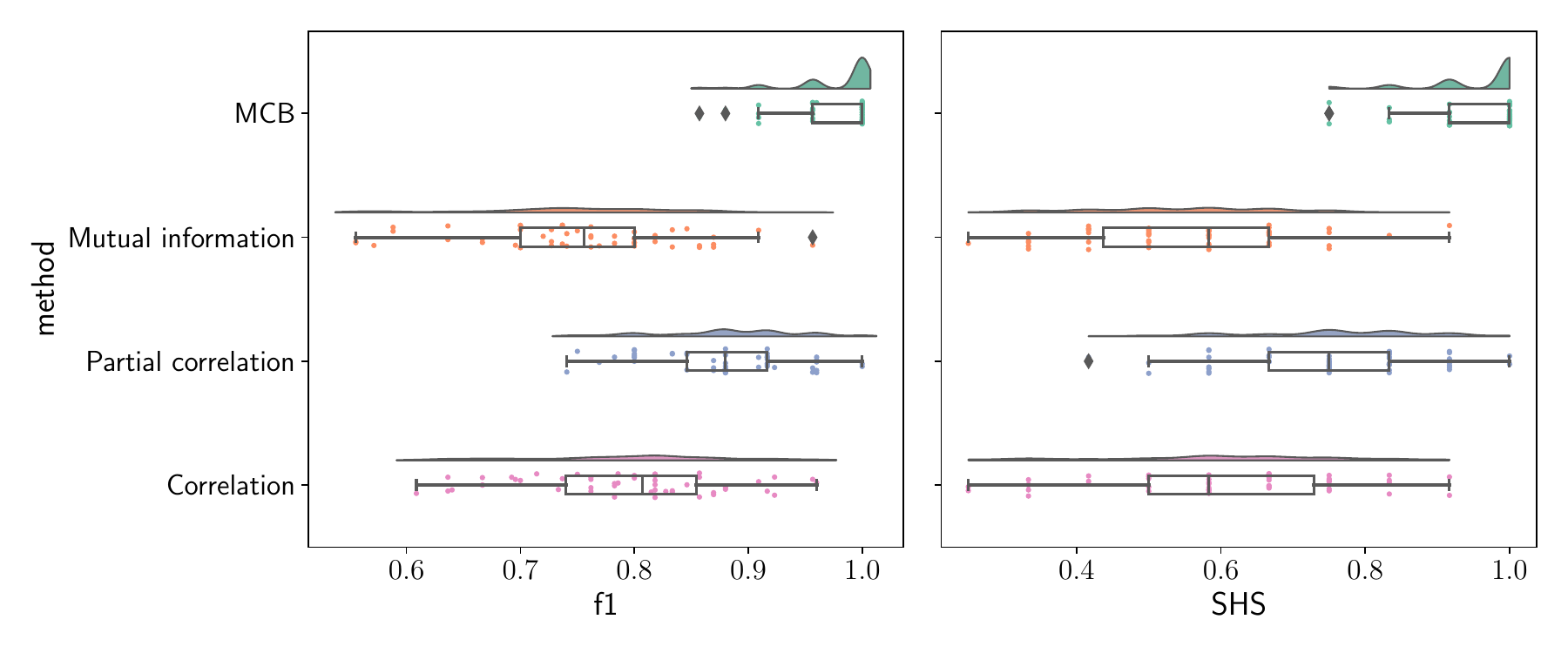}
    \caption{Raincloud plots of the distributions of (left) F1 score and (right) SHS, obtained by the tested methods over $50$ synthetic data sets. For readability, we omit the results by DTF and PDC.}
    \label{fig:synthreadable}
\end{figure}

% \spara{Results.}
\Cref{fig:synthreadable} depicts the results obtained by the considered methods over the $50$ synthetic data sets.
Here, we omit DTF and PDC for readability reasons, as their performance is quite low (\Cref{subapp:addresults1} shows the complete figure).
Our methodology outperforms the baselines on both metrics.
Specifically, the kernel density estimate (KDE) corresponding to our method is highly concentrated around the maximum value. 
Conversely, the KDEs associated with mutual information and Pearson's correlation exhibit high dispersion. 
The results obtained from these two methods are statistically equivalent, as expected. 
Finally, the second most effective method is partial correlation. In a linear setting, this method benefits from information regarding conditional dependence.
% !TEX root =  ../main.tex
\section{Analysis of resting-state fMRI data}\label{sec:results}
We study the data set introduced by \citet{termenon2016reliability}, which encompasses a large sample of rs-fMRI data, publicly released as part of the HCP~\citep{smith2013resting}.
The data set comprises rs-fMRI recordings from $100$ healthy adults, with each subject undergoing two rs-fMRI acquisitions on separate days. 
During the acquisition, subjects were instructed to lie with their eyes open, maintaining a relaxed fixation on a white cross against a dark background. 
They were asked to keep their minds wandering and remain awake throughout the session.

Accordingly, we have $S = 100$ individual data sets, for two separate days.
The parcellation scheme divides the brain into $89$ ROIs. 
We separately analyze the ROIs of the left hemisphere and those of the right hemisphere. 
This approach prevents signal averaging across brain regions corresponding to different hemispheres. 
Consequently, each hemisphere contains $K = 45$ ROIs, with the vermis region being common to both.
Additionally, each acquisition session has a duration of $14$ minutes and $24$ seconds, which results in a total of $N = 1200$ timestamps. 
For further details regarding the data source, we refer the interested reader to \citet{termenon2016reliability}.

\spara{Analysis of the MCBs.}
We apply our methodology and the baseline methods considered in \Cref{sec:synth_data} to this dataset. 
To retrieve the brain connectivity backbone according to the baselines, we use the same approach as in \Cref{sec:synth_data}, which is detailed in \Cref{subapp:addresults1}. 
For the multiscale decomposition performed by MS-CASTLE, we use the SWT with Daubechies wavelets with filter length equal to $10$, consistently with the approach used by \citet{termenon2016reliability}.
The set of $p$-persistent arcs \candidateuniverse is constructed by selecting arcs that are present in more than $p=80$ individuals on each day of the acquisition.
This criterion is also used within the application of the baseline methods to decide whether an arc has to be added to the connectivity backbone. 
Furthermore, we set the threshold parameter to $\tau=0.2$, for all the methods, after having explored the effect of $\tau$ on the cardinality of the arc sets of the individual multiscale DAGs (see \Cref{subapp:addresults2}).
Finally, we evaluate the statistical significance of the arcs within the MCB retrieved by our method via a bootstrap with resampling.
More in detail, starting from the $200$ multiscale DAGs learned by MS-CASTLE, we create $100$ samples, where each sample consists of $200$ multiscale DAGs obtained using bootstrap. 
At this point, we retrieve the MCB corresponding to each sample, and we retain the arcs that are associated with weights statistically different from zero at $10\%$ level within the $\{\candidatecausalgraph\}$.

Baseline methods retrieve dense networks for both the right and left hemispheres, and it is challenging to draw any conclusion from the results.
Due to space constraints, the corresponding connectivity backbones are given in \Cref{subapp:addresults3}. 

Conversely, our methodology produces sparse MCBs for both hemispheres, provided in \Cref{fig:MCBsfull} in \Cref{subapp:addresults3} for space reasons.
Our findings suggest that finer scales are characterized by more complex structures.
Here ROIs associated with different functionalities interact.
Conversely, at coarser scales causal interactions mainly occur between ROIs associated with the same functionalities.
Furthermore, as expected nodes related to the processing of external stimuli (i.e., visual, auditory, speech, and language) appear and interact at the first three scales.
Conversely, nodes associated with higher-level cognitive functions and emotions appear throughout the considered scales, as expected in the resting state.

% Specifically, referring to the networks provided in \Cref{fig:MCBsfull}, our findings show that:
Specifically, our findings are the following:

\begin{squishlisttight}
    \item The occipital lobe consistently plays a central role within the visual network (VN) across various frequency scales. This suggests its fundamental role in visual processing \citep{utevsky2014precuneus,cunningham2017structural}.
    \item The superior parietal gyrus, crucial for sensory integration and attention, has widespread impacts. It influences the precuneus in the default mode network (DMN), indicating a link between sensory integration and self-awareness \citep{freedman2018integrative,tamietto2015once}.
    \item The postcentral gyrus, primarily involved in somatosensory information processing, is a key node. At finer scales, it connects to the precentral gyrus, which governs voluntary motor movements, and further activates the middle frontal gyrus, involved in high-level cognitive functions \citep{johansen2002attention,brown2001motor}. 
    \item Nodes in the frontal regions (e.g., orbital surface of the inferior frontal gyrus, lateral surface of the superior frontal gyrus) have significant impacts. They govern emotion processing, cognitive flexibility, decision-making, and cognitive control \citep{boisgueheneuc2006functions,rolls2000orbitofrontal}.
    \item The middle temporal gyrus, involved in various functions like language, memory, and cognition \citep{davey2016exploring}, influences other temporal regions, including the superior temporal gyrus and the temporal pole \citep{allison2000social,diveica2021establishing}. Additionally, several works in neuroscience point out a link between abnormalities in these brain regions and the autism disorder spectrum (see \citealp{jou2010enlarged,sato2017reduced} and references therein).
    \item The nature of brain connectivity is dynamic, with certain nodes serving as common causes across different networks and frequency scales. The multiscale causal analysis points out the presence of bi-directional connections (e.g., between the precentral gyrus and the postcentral gyrus), and effectively locates the occurrence of the change of direction of the arcs in the frequency domain.
    Additionally, it suggests two regimes in causal interactions: (i) at higher frequencies, sensory information is a key driver in the network \citep{mantini2007electrophysiological,smith2009correspondence}, together with the high-level cognitive function of the lateral surface of the superior frontal gyrus; (ii) at lower ones the shared brain activity is driven by the nodes of the DMN, specifically the lateral surface of the superior frontal gyrus and the precuneus, which contribute to processes related to self-awareness, introspection, and self-referential thinking \citep{raichle2001default,uddin2008network}.
\end{squishlisttight}
The detailed analysis organized per scales is given in \Cref{subapp:addresults3}.

\spara{Comparison with the SCBs.}
An interesting point concerns the comparison between the MCBs and the SCBs that can be obtained by the application of the backbone search procedure (equipped with bootstrap with resampling) to subject-specific single-scale causal graphs.

% \begin{figure}[htb]
%     \hspace{15mm}
%     \centering
%     \begin{minipage}[b]{.35\textwidth}
%         \centering
%          \includegraphics[width=1.0\textwidth]{figs/results/1.3_single_scale_left.pdf}
%          \subcaption{Left}
%          \label{subfig:ssleft}
%          \vspace{3mm}
%      \end{minipage}
%      \vspace{-2mm}
%      \hfill
%      \begin{minipage}[b]{.35\textwidth}
%         \centering
%          \includegraphics[width=1.0\textwidth]{figs/results/1.3_single_scale_right.pdf}
%          \subcaption{Right}
%          \vspace{3mm}
%          \label{subfig:ssright}
%      \end{minipage}
%      \hspace{15mm}
%      \vspace{-2mm}
%     \caption{SCBs for (\subref{subfig:ssleft}) the left and (\subref{subfig:ssright}) the right hemispheres. Color-filled nodes are the main drivers within the MCBs in \Cref{fig:MCBs}.}
%     \label{fig:SCBs}
% \end{figure}

% \Cref{fig:SCBs} shows the SCBs for both hemispheres retrieved by applying the single-scale counterpart of MS-CASTLE, i.e., SS-CASTLE~\citep{d'acunto2023multiscale}. 
The SCBs provided in \Cref{fig:SCBs} in \Cref{subapp:addresults4} confirm the importance of the regions highlighted by the multiscale analysis, which also appear to be the main drivers in this analysis. 
However, while the MCBs allow us to study causal structures in distinct frequency bands, thus enabling us to locate the predominance of certain brain activities over others in the frequency domain and detect structural changes, in the SCBs the information is aggregated. 
As an example, consider the causal connection from the superior parietal gyrus (node $26$) to the precuneus (node $30$) in \Cref{subfig:ssleft}.
Looking at the MCBs in \Cref{subfig:left1,subfig:left2,subfig:left3,subfig:left4,subfig:left5}, we see that this connection only occurs at the first three time scales, i.e., frequencies greater than $0.087$ Hz. 
Moreover, the aggregation of information within the SCBs also leads to the loss of certain connections, such as the bidirectional one between the precentral gyrus (node $0$) and the postcentral gyrus (node $25$).
These regions belong to the sensory-motor network \citep{mantini2007electrophysiological,smith2009correspondence}, and the bi-directionality of this causal relation confirms that they collaborate to control and sense movements.

\spara{Causal fingerprinting.}
Finally, an intriguing point concerns the individual variability of multiscale causal structures, that we name \emph{causal fingerprinting}.
This concept relates to the well-studied functional connectivity fingerprinting, i.e., the usefulness of functional connectivity in identifying subjects within large groups~\citep{miranda2014connectotyping,finn2015functional,liu2018chronnectome,elliott2019general}.
% Here, we test three different hypotheses, whose statistical significance has been assessed by the two-sample Cram\'{e}r-von Mises \citep{anderson1962distribution} tests at $0.1\%$ level, after applying the Bonferroni correction~\citep{bonferroni1936teoria}.
% Let us consider $G^i_{h,t}$ the MDAG corresponding to the hemisphere $h$ of subject $i$ at acquisition $t$, where $h \in \{\mathrm{left}, \mathrm{right}\}$, $t \in \{t_1,t_2\}$, and $i \in [S]$.
Let us indicate with $G^i_{h,t}$ the individual multiscale DAG learned by MS-CASTLE, and corresponding to the hemisphere $h$ of subject $i$ at acquisition $t$, where $h \in \{\mathrm{left}, \mathrm{right}\}$, $t \in \{t_1,t_2\}$, and $i \in [S]$.
The tested hypothesis is that the similarity between \(G^i_{h,t_1}\) and \(G^i_{h,t_2}\) is statistically different from the similarity between \(G^i_{h,t_1}\) and \(G^j_{h,t_1}\) for \(j \in [S-1]\).
Mathematically, this can be expressed as \(\mathrm{H}_0: \mathrm{J}(G^i_{h,t_1}, G^i_{h,t_2}) = \mathrm{J}(G^i_{h,t_1}, G^j_{h,t_1})\), where \(\mathrm{J}\) represents the Jaccard score used for measuring the similarity between individual multiscale DAGs \citep{jaccard1912distribution}.
Statistical significance is assessed by the two-sample Cram\'{e}r-von Mises \citep{anderson1962distribution} test at $0.1\%$ level, after applying the Bonferroni correction~\citep{bonferroni1936teoria}.

\begin{figure}[t]
\vspace{-.5\baselineskip}
\begin{tabular}{@{} c p{.4\textwidth} @{}}
    \includegraphics[scale=0.33,valign=T]{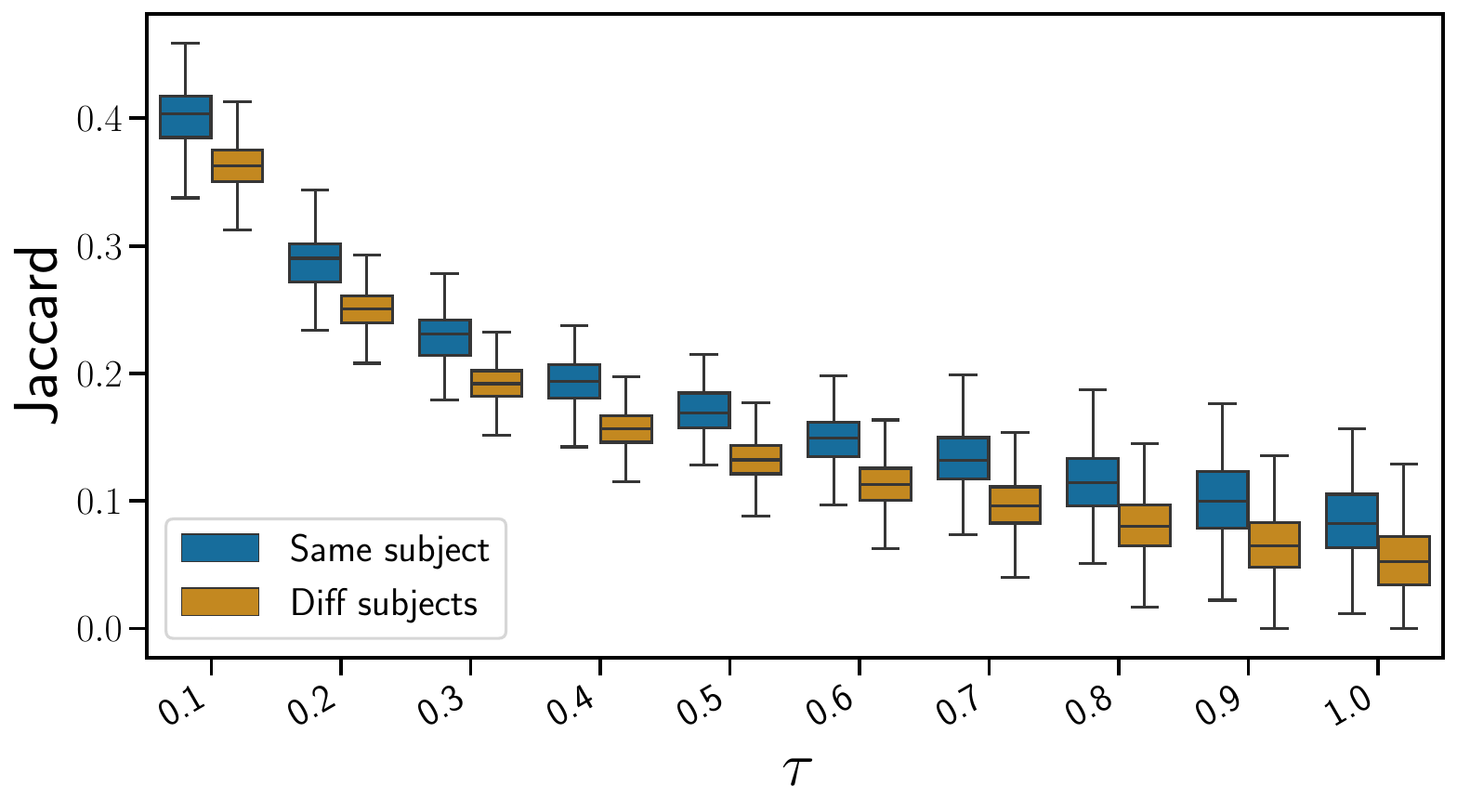}
    &
    \caption{Comparison of the Jaccard score between the individual multiscale DAGs corresponding to the same hemisphere and the \emph{same subject} at different timestamps (acquisitions) $t_1$ and $t_2$, to the Jaccard score between the individual multiscale DAGs corresponding to the same hemisphere of \emph{different subjects} at the same timestamp. The analysis is conducted at different hard-thresholding levels $\tau \in (0,1]$.}
    \label{fig:Jacc}
    \vspace{-\baselineskip}
\end{tabular}
\end{figure}

The results are depicted in Figure \ref{fig:Jacc}. The individual multiscale DAGs for the same subject are indeed more similar than those of different subjects across various hard-thresholding levels \(\tau\). 
The null hypothesis is rejected at each \(\tau\).
Thus, individual multiscale DAGs can effectively differentiate between different subjects within large groups, supporting the existence of causal fingerprinting.

% In the second case, the null hypothesis aims to determine if the similarity between \(G^i_{h,t_1}\) and \(G^i_{h,t_2}\) is statistically different from the similarity between \(G^i_{h,t_1}\) and \(G^i_{g,t_1}\) for \(g \in \{\mathrm{left}, \mathrm{right}\}\), i.e., \(\mathrm{H}_0\coloneqq \mathrm{J}(G^i_{h,t_1}, G^i_{h,t_2}) = \mathrm{J}(G^i_{h,t_1}, G^i_{g,t_2})\). 
% The results are displayed in Figure \ref{fig:Jacc} (center plot). 
% For a given subject, the MDAGs from the same hemisphere at different timestamps exhibit more similarity than those between different hemispheres, across various hard-thresholding levels \(\tau\). 
% The null hypothesis is rejected by the two-sample test for each \(\tau\), suggesting that left and right hemispheres should be analyzed separately.

% Lastly, in the third case, the null hypothesis examines whether the similarity between \(G^i_{h,t_1}\) and \(G^j_{h,t_1}\) is statistically different from the similarity between \(G^i_{h,t_2}\) and \(G^j_{h,t_2}\) for \(i,j \in [S]\), i.e., \(\mathrm{H}_0\coloneqq \mathrm{J}(G^i_{h,t_1}, G^j_{h,t_1}) = \mathrm{J}(G^i_{h,t_2}, G^j_{h,t_2})\). 
% The results are depicted in Figure \ref{fig:Jacc} (right plot). 
% Here, we observe that MDAGs from different subjects exhibit consistent similarities across different acquisition days \(t_1\) and \(t_2\). 
% The null hypothesis is not rejected up to \(\tau=0.4\), indicating that the similarity between subjects is not dependent on the day of acquisition.

% !TEX root =  ../main.tex
\section{Conclusion \textcolor{black}{and future work}}\label{sec:conclusions}
In this paper, we investigate the discovery of an MCB in brain dynamics across individuals. 
We propose a principled methodology that leverages recent advances in multiscale causal structure learning that optimizes the balance between model fit and complexity \textcolor{black}{by adopting the MDL}.
\textcolor{black}{Being a general principle, we need to tailor the MDL principle to the application's needs. 
From this perspective, \Cref{alg:search} is the result of modeling choices and assumptions appropriate for resting-state fMRI data.
In this regard, the usage of a linear model adheres to \Cref{a1}, which is supported by previous work on resting-state fMRI data (\citealp{hlinka2011functional,garg2013gaussian,novelli2022mathematical}, and references therein).
}
Our approach outperforms a baseline based on canonical functional connectivity in learning the MCB, as demonstrated on synthetic data. 
Applying this method to real-world resting-state fMRI data reveals sparse MCBs in both the left and right brain hemispheres. The multiscale nature of our approach allows us to extract that high-level cognitive functions drive causal dynamics at low frequencies, while sensory processing nodes are relevant drivers at higher frequencies.
Our analysis confirms the presence of a causal fingerprinting of brain connectivity among individuals, from a causal standpoint.

%practical implications
\textcolor{black}{We believe that our approach can be useful as an input for downstream computational models, to augment existing fingerprinting approaches \citep{van2021makes} and markers of changes in personal backbone through aging, neurodegeneration, and other pathological functional patterns \citep{cole2017predicting,pievani2014brain,hohenfeld2018resting}.}
%future work
Future work should be devoted to studying MCBs by using data from individuals performing various tasks and individuals with neurological disorders. 
\textcolor{black}{In these cases, our modeling choices need to be revised to meet the application needs (see, e.g., \citealp{huang2020causal}).
However, comparing our method to state-of-the-art functional connectivity methods on task-based fMRI data would still be insightful, as the vast majority of these latter methods are based on Pearson's correlation.}
Discovering MCBs in these contexts is crucial for understanding how brain dynamics operate when the brain is subject to stimuli, and how they are affected by neurological disorders. 
Such studies can lead to insights that may facilitate the development of targeted interventions and therapies for cognitive enhancement or the treatment of neurological conditions.

\bibliography{bibliography}

\appendix
%% TeX root main.tex

\section{Proof of Lemma \ref{lem:waveletgauss}}\label{app:lemma1}

\waveletgauss*
\begin{proof}
    Let us consider the signal $\signal_i \in \real^N$, and define $\mathbf{c}_i^{s,0}=\tilde{\mathbf{x}}_i^{s,0}=\signal_i$.
    Consider an orthogonal wavelet family defined by the high-pass and low-pass filters, $\mathbf{h} \in \real^L$ and $\mathbf{g} \in \real^L$, respectively.
    Hence we have that
    \begin{equation*}
        \begin{aligned}
            \tilde{x}_i^{s,j}[n] &= \sum_{l=0}^{L-1} h[l]c_i^{s,j-1}[n+l], \\
            c_i^{s,j}[n] &= \sum_{l=0}^{L-1} g[l]\tilde{x}_i^{s,j-1}[n+l].
        \end{aligned}
    \end{equation*}

    Thus, given Assumption \ref{a1}, $\tilde{x}_i^{s,j}[n]$ and $c_i^{s,j}[n]$ are distributed according to a zero-mean Gaussian distribution, since linear combinations of i.i.d. Gaussian variables from the $i$-th dimension of $N(\mathbf{0}, \boldsymbol{\Omega}^s)$.
\end{proof}

%% TeX root main.tex

\section{Computational cost}\label{app:cost}
In our setting $N>K$ (cf, \Cref{sec:results}). 
At each step, we have to update the family score in \eqref{eq:familyscore} for the child node of the last added $p$-persistent arc.
Let's say that the child node is $r$, and it has $K^{s,j}<K-1$ parents for the scale $j$ and individual $s$.
% According to \Cref{app:modelselection}, for evaluating the candidate $p$-persistent arc $u_{lr}^j$, we need to compute the least-squares estimates in \eqref{eq:regression}.
\textcolor{black}{According to \Cref{sec:learningMCB}, to evaluate the candidate $p$-persistent arc $u_{lr}^j$, we need to compute the least-squares estimates in \eqref{eq:regression}.}
The cost of evaluating the least-squares estimates of the regression coefficients for all the candidate $p$-persistent arcs $u_{lr}^j$ ending in $r$ is asymptotically $\mathcal{O}(N(K^{s,j}+1)^2)$.
Indeed, we first compute the pseudo-inverse of $\multidataT \in \real^{N \times (K^{s,j}+1)}$ using SVD decomposition, requiring $\mathcal{O}(N(K^{s,j}+1)^{2})$ flops.
Then, we multiply the pseudo-inverse by $\multisignal_r \in \real^N$, requiring a cheaper cost, i.e., $\mathcal{O}(2(K^{s,j}+1)N)$.
Since we can parallelize over the individuals and over the time scales, the overall cost for the update of the family score for each candidate $p$-persistent arc is $\mathcal{O}(N(K^\prime+1)^2)$, where $K^\prime=\max_{s,j} K^{s,j}$.
Now, let us define $U=\max_{j} |\candidateuniverse^j|$.
Since we might have more than one $p$-persistent arc to evaluate, the total cost for evaluating all the least-squares estimated is asymptotically upper-bounded by $\mathcal{O}(UN(K^\prime+1)^2)$.

At this point, we compute $\Delta_{RIC}^j$ for the candidate $p$-persistent arcs $u_{lr}^j$, which requires evaluating \eqref{eq:familyscore} given the least-squares estimates.
In particular, here the most demanding operation is the computation of the quadratic term, which is upper-bounded by $\mathcal{O}((2N+3)K^\prime)$.
Again, considering that we might have multiple $p$-persistent arcs, the overall cost for this step is $\mathcal{O}(U(2N+3)K^\prime)$.

Subsequently, we select the arc having the minimum absolute $\Delta_{RIC}^j$, with a cost linear in $|\candidateuniverse^j|$.
Finally, we check the acyclicity of the $j$-th layer of MCB with arc set $\estimatedbackbonesetj$, requiring $\mathcal{O}(|\estimatedbackbonesetj|)$ flops.
Based on this analysis, the most demanding step is the computation of least-squares estimates for updating the family scores, which at the beginning is done for all the $p$-persistent arcs in $\candidateuniverse^j$, in parallel over $j \in [J]$.
%% TeX root main.tex

\section{Hyper-parameters}\label{app:hyperparams}

\begin{table}[h!]
    \centering
    \caption{The table shows the values for the relevant hyper-parameters used in the results depicted in \Cref{fig:synthreadable,fig:synthfull}.}
    \resizebox{.8\textwidth}{!}{%
    \begin{tabular}{cccc}
         Method&  Persistence $p$&  Hard-thresholding $\tau$& $\ell_1$-regularization strength, $\lambda$\\
         MCB&  65&  0.15& 0.01\\
         Mutual information&  60&  0.05& N/A\\
         Pearson's correlation&  65&  0.15& N/A\\
         Partial correlation&  65&  0.25& N/A\\
         DTF&  60&  0.0& N/A\\
 PDC& 60& 0.0&N/A\\
    \end{tabular}}
    \label{tab:hyperparams}
\end{table}

The values for the relevant hyper-parameters used in the empirical assessment on synthetic data are provided in \Cref{tab:hyperparams}.
The $\ell_1$-regularization strength is only used by SS-CASTLE.
%% TeX root main.tex

\section{Metrics}\label{app:metrics}

We define the considered metrics similarly to \citet{zheng2018dags}. 
Specifically, we denote:
\begin{squishlist}
    \item $\mathrm{cp}$ as the count of true positive arcs, i.e., the number of arcs present in the ground truth.
    \item $\mathrm{cn}$ as the count of true negative arcs, i.e., the number of arcs that are absent in the ground truth.
    \item $\mathrm{tnnz}$ as the total number of arcs in the ground truth.
    \item $\mathrm{nnz}$ as the total number of estimated arcs.
    \item $\mathrm{tp}$ as the count of true positives, i.e., the number of estimated arcs present in the ground truth and with the correct direction.
    \item $\mathrm{r}$ as the count of reversed arcs, i.e., the number of learned arcs present in the ground truth but with the opposite direction.
    \item $\mathrm{fp}$ as the count of false positives, i.e., the number of learned extra arcs not present in the undirected skeleton of the ground truth.
    \item $\mathrm{e}$ as the count of missing arcs, i.e., the number of arcs in the skeleton of the learned graph that are extra compared to the skeleton of the ground truth.
    \item $\mathrm{m}$ as the count of extra arcs, i.e., the number of arcs in the skeleton of the learned graph that are missing compared to the skeleton of the ground truth.  
\end{squishlist}

With these definitions in place, we can calculate the following metrics:
\begin{squishlist}
    \item False Discovery Rate (FDR) is given by $\mathrm{fdr}=(\mathrm{r} + \mathrm{fp})/\mathrm{nnz}$.
    \item True Positive Rate (TPR) is calculated as $\mathrm{tpr}=\mathrm{tp}/\mathrm{cp}$.
    \item False Positive Rate (FPR) is determined as $\mathrm{fpr}=(\mathrm{r}+\mathrm{fp})/\mathrm{cn}$.    
\end{squishlist}

Consequently, the F1-score is equal to $2\cdot\frac{(1-\mathrm{fdr})\mathrm{tpr}}{1-\mathrm{fdr}+\mathrm{tpr}}$, the normalized Structural Hamming Distance (nSHD) is computed as $\frac{(\mathrm{r}+\mathrm{m}+\mathrm{e})}{\mathrm{tnnz}}$, and the structural Hamming similarity is $\mathrm{SHS}=1-\mathrm{nSHD}$.
%% TeX root main.tex

\section{Additional results}\label{app:addresults}

\begin{figure}
    \centering
    \includegraphics[width=\textwidth]{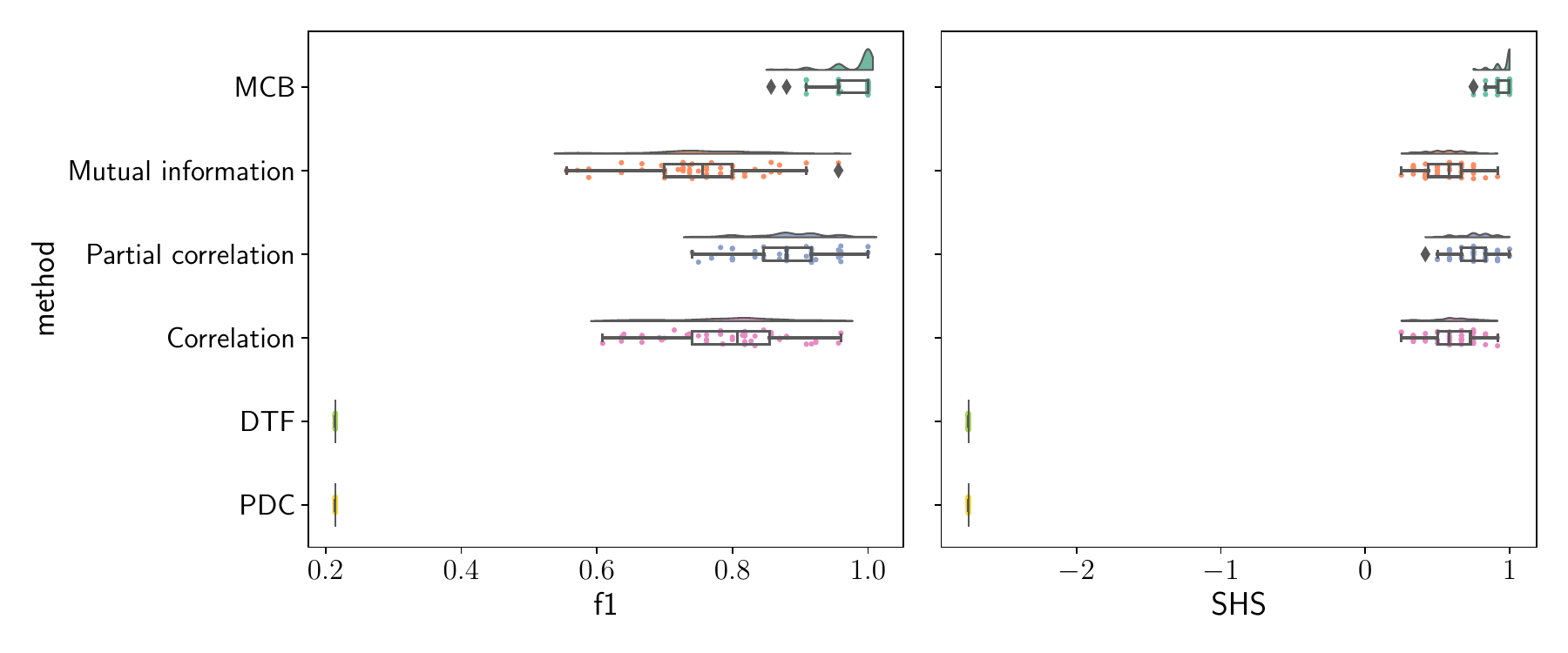}
    \caption{Rain plots depicting the (left) F1 score and (right) SHS, obtained by the considered methods over $50$ synthetic data sets.}
    \label{fig:synthfull}
\end{figure}

\subsection{Synthetic data}\label{subapp:addresults1}

Here we dive into the procedure used for testing the considered methods on synthetic data, and we provide additional discussion about the comparison.

% \spara{Data generation.}
% We consider $(S,K,N)=(10,100,1200)$, and w.l.o.g., $J=1$.
% Indeed, since the time scales are independent, the performance of our method is not affected by the number of time scales.
% We generate a strictly lower triangular masking matrix $\backbonemask \in \{0,1\}^{K \times K}$, where $[\backbonemask]_{ij} \sim B(0.25)$ and $B$ is the Bernoulli distribution. 
% The non-zero entries of this matrix correspond to the arcs in the causal backbone.
% Similarly, we generate strictly lower triangular masking matrices $\idiosyncraticmask \in \{0,1\}^{K \times K}$, where $[\idiosyncraticmask]_{ij} \sim B(0.5)$, with $s \in [S]$.
% The non-zero entries of these matrices, which are not in the backbone, represent the idiosyncratic connection for the individuals.
% Hence, $\summask = \backbonemask + \idiosyncraticmask$ is the adiaciency matrix corresponding to $G^{s}$ for individual $s$.
% At this point, we sample $\weights \in \real^{K \times K}$ from a uniform $U(-1,1)$, thus obtaining the causal matrix for each individual $\causalstructure = \summask \circ \weights, \, \forall s \in [S]$, representing the weights of the arcs of $G^{s}$.
% Finally, we generate data according to \Cref{eq:multiSEM}.
% We build $50$ data sets according to this procedure.

\spara{Baselines.}
As shown in \Cref{sec:synth_data}, we compare the proposed methodology with widely used measures of functional connectivity in terms of F1 score and structural Hamming similarity.
In detail, among the connectivity measures outlined in \Cref{sec:introduction}, we consider \emph{(i)} DTF and PDC among the existing directed measures, and \emph{(ii)} Pearson's correlation, partial correlation, and mutual information among the undirected ones.
We apply these connectivity measures to each subject $s \in [S]$ in order to retrieve individual brain connectivity networks, where each node corresponds to one of the ten synthetically generated time series.

In the case of undirected measures, the resulting individual brain connectivity networks are undirected.
For subject $s$, the weight of the arc between nodes $l$ and $m$ represents the value of the considered connectivity measure when it is computed by taking as input the time series corresponding to nodes $l$ and $m$ associated with subject $s$.
In the case of directed measures, the obtained networks are directed, but not necessarily acyclic.
In addition, given two time series, DTF and PDC evaluate the statistical dependence by analyzing the spectral Granger causality.
In brief, DTF and PDC consider $F \in \nat$ different frequencies, and for each $f \in [F]$, they return the measure of statistical dependence between the time series (refer to~\citealp{kaminski1991new,baccala2001partial} for details).
This means that, by considering all the $K$ time series, we obtain $K$ by $K$ connectivity matrices for every $f$.
Let us denote this matrix with $\mathbf{D}^{s,f}$.
Now, in order to obtain the weight of the arcs for the individual brain connectivity networks, we stack the connectivity matrices $\mathbf{D}^{s,f}$ along the frequency dimension, i.e., along $f$, thus obtaining a $K$ by $K$ by $F$ tensor $\{\mathbf{D}^{s,f}\}_{f \in [F]}$.
Hence, to aggregate this information in a matrix $\widetilde{\mathbf{D}}^s \in \real^{K \times K}$, we compute the $\ell_2$-norm along the tube dimension of the tensor corresponding to the frequency index, i.e., $f$.
At this point, we obtain the weighted adjacency of the brain connectivity network for the $s$-th individual (and thus the weights for the arcs), by computing $\widehat{\mathbf{D}}^s=\mathbf{R}^s \widetilde{\mathbf{D}}^s \mathbf{R}^s$, where $\mathbf{R}^s \in \real^{K\times K}$ is the diagonal matrix with entries $[\widetilde{\mathbf{D}}^s]_{ii}^{-1/2}, \, \forall i \in [K]$.

Once we have obtained all the brain connectivity networks, we apply hard-thresholding to retain only those arcs whose absolute weights are greater than a hyper-parameter $\tau \in \real_{+}$.
The same hyper-parameter is also used for pruning the single-scale causal DAGs returned by the single-scale counterpart of MS-CASTLE, i.e., SS-CASTLE \citep{d'acunto2023multiscale}.

At this point, for each baseline, we consider as an estimate of $\backboneset$ the set of connections that occur in more than $p$ individual brain connectivity networks, where $p<S, \, p \in \nat$.
Our approach, instead, uses the same hyper-parameter $p$ to individuate the $p$-persistent arcs in Definition \ref{def:ppersistentArc} composing $\candidateuniverse$.
Then, the score-based search procedure in \Cref{sec:learningMCB} is run over the pruned single-scale individual causal DAGs to retrieve the $\backboneset$.
Details concerning the fine-tuning of $\tau$ and $p$ are given in \Cref{app:hyperparams}.

\spara{Additional comparison with baselines.} 
\Cref{fig:synthfull} depicts the results obtained by the considered methods over $50$ synthetic data sets.
In addition to the discussion provided in \Cref{sec:synth_data}, here we see that both DTF and PDC fail in backbone discovery, returning the fully connected single-scale causal backbone (SCB) in every case.
We believe that there are two reasons behind this behavior. 
Firstly, the underlying ground truth causal structure contains only instantaneous connections, while DTF and PDC are based on the spectral formulation of Granger causality, which pertains to lagged causal connections.
Secondly, according to both methods, a relationship between two signals is identified if the measured value is different from zero for at least one frequency.
Consequently, noisy estimates can lead to the presence of spurious arcs.

\subsection{Analysis of the cardinality of individual multiscale DAGs}\label{subapp:addresults2}

% \begin{figure}
%     \centering
%     \includegraphics[width=.8\textwidth]{figs/addresults/1.0_hard_thresholding.pdf}
%     \caption{The figure depicts the behavior of the cardinality of the individual multiscale causal DAGs for the two hemispheres and different days of acquisition along $\tau$, normalized by the number of arcs of the fully connected individual multiscale causal DAG.}
%     \label{fig:cardinality}
% \end{figure}

\begin{figure}[h]
\vspace{-.5\baselineskip}
\begin{tabular}{@{} c p{.25\textwidth} @{}}
    \includegraphics[scale=0.43,valign=T]{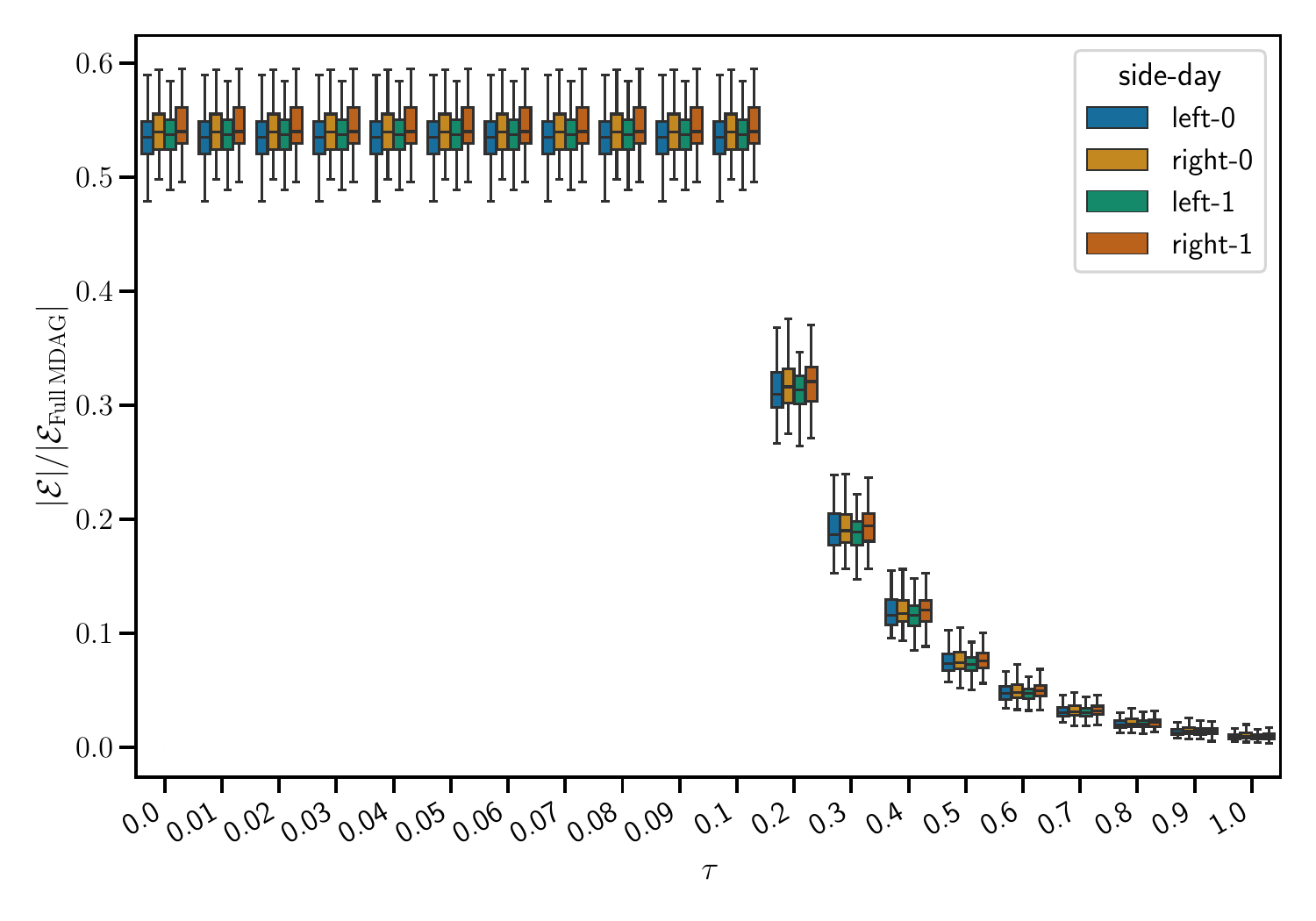}
    &
    \caption{The figure depicts the behavior of the cardinality of the individual multiscale causal DAGs for the two hemispheres and different days of acquisition along $\tau$, normalized by the number of arcs of the fully connected individual multiscale causal DAG.}
    \label{fig:cardinality}
    \vspace{-\baselineskip}
\end{tabular}
\end{figure}

\Cref{fig:cardinality} shows the impact of the hard-thresholding level $\tau$ on the cardinality of the arc set of the individual multiscale causal DAGs for the two hemispheres and different days of acquisition.
To facilitate the interpretation of the plot, we normalize the cardinality of the arc sets by the number of arcs of the fully connected individual multiscale causal DAG, i.e., $(J\cdot N \cdot(N-1))/2$.
As we see, $\tau=0.2$ is the first value (among those considered) that effectively reduces the cardinality of the arc sets.
Accordingly, we use this value in the analysis given in \Cref{sec:results}.

\subsection{Comparison with baselines and analysis of the resulting MCBs by scale}\label{subapp:addresults3}

\begin{sidewaystable}
\centering
\caption{The table shows the number of the node corresponding to each ROI. In addition, we report the macro-afferent regions.}

\resizebox{.8\textwidth}{!}{%
\begin{tabular}{|c|c|c|l|} \hline 

\textbf{Node  Number} & \textbf{Region}& \textbf{Macro-afferent Region} & \textbf{Related Functions} \\ \hline 
0 & Precentral gyrus & Central  &Motor network \citep{penfield1937somatic}\\ \hline 
1 & Superior frontal gyrus & Frontal - Lateral  &Cognitive functions, attention \citep{boisgueheneuc2006functions}\\ \hline 
2 & Superior frontal gyrus & Frontal - Orbital  &Emotion, working memory \citep{zhao2020reduced,hu2016right}\\ \hline 
3 & Middle frontal gyrus & Frontal - Lateral  &Working memory, attention, memory \citep{ridderinkhof2004role, koechlin2003architecture}\\ \hline 
4 & Middle frontal gyrus & Frontal - Orbital  &Emotion, decision making \citep{zhao2020reduced} \\ \hline 
5 & Inferior frontal opercular gyrus & Frontal - Lateral  &Speech and language processing \citep{price2012review}\\ \hline 
6 & Inferior frontal triangular gyrus & Frontal - Lateral  &Speech and language processing, working memory \citep{price2012review}\\ \hline 
7 & Inferior frontal gyrus & Frontal - Orbital  &Emotion, decision making \citep{rolls2000orbitofrontal}\\ \hline 
8 & Rolandic operculum & Central  &Speech and language processing, working memory, attention \citep{10.1093/acprof:oso/9780195177640.003.0014}\\ \hline 
9 & Supplementary Motor Area & Frontal - Medial  &Motor network \citep{nachev2008functional}\\ \hline 
10 & Olfactory cortex & Frontal - Orbital  &Olfaction \citep{gottfried2010central}\\ \hline 
11 & Superior frontal gyrus & Frontal - Lateral  &Default mode network \citep{uddin2008network}\\ \hline 
12 & Superior frontal and rectus gyri & Frontal - Orbital  &Emotion, decision making \citep{ray2012anatomical}\\ \hline 
13 & Insula & Insula  &Emotion, self-awareness \citep{craig2009you}\\ \hline 
14 & Anterior cingulate and paracingulate gyri & Limbic  &Emotion, decision making \citep{bush2000cognitive}\\ \hline 
15 & Median cingulate and paracingulate gyri & Limbic  &Emotion, pain processing \citep{vogt2005pain}\\ \hline 
16 & Posterior cingulate gyrus & Limbic  &Default mode network \citep{buckner2008brain}\\ \hline 
17 & Hippocampus & Limbic  &Memory formation, memory retrieval \citep{eichenbaum2017role}\\ \hline 
18 & Parahippocampal gyrus & Limbic  &Memory formation, memory retrieval \citep{gandolla2014re} \\ \hline 
19 & Amygdala & Subcortical Gray Nuclei  &Emotional processing \citep{ledoux2007amygdala}\\ \hline 
20 & Calcarine fissure and surrounding cortex & Occipital - Medial and inferior  &Primary visual \citep{wandell2011imaging}\\ \hline 
21 & Cuneus & Occipital - Medial and inferior  &Visual processing \citep{wandell2011imaging}\\ \hline 
22 & Lingual gyrus & Occipital - Medial and inferior  &Object recognition \citep{grill2004human} \\ \hline 
23 & Occipital gyrus & Occipital - Lateral  &Visual processing \citep{wandell2011imaging}\\ \hline 
24 & Fusiform gyrus & Occipital - Medial and inferior  &Face and object recognition \citep{kanwisher1997fusiform}\\ \hline 
25 & Postcentral gyrus & Central  &Somatosensory network \citep{iwamura1994bilateral}\\ \hline 
26 & Superior parietal gyrus & Parietal - Lateral  &Sensory integration, attention \citep{corbetta2002control}\\ \hline 
27 & Inferior parietal gyrus & Parietal - Lateral  &Default mode network, attention \citep{caspers2006human}\\ \hline 
28 & Supramarginal gyrus & Parietal - Lateral  &Language processing, working memory, memory retrieval \citep{rushworth2003left}\\ \hline 
29 & Angular gyrus & Parietal - Lateral  &Default mode network \citep{seghier2013angular}\\ \hline 
30 & Precuneus & Parietal - Medial  &Default mode network \citep{raichle2001default}\\ \hline 
31 & Paracentral lobule & Frontal - Medial  &Motor network \citep{picard1996motor}\\ \hline 
32 & Caudate nucleus & Subcortical Gray Nuclei  &Motor control, cognitive control \citep{graybiel2000basal}\\ \hline 
33 & Putamen & Subcortical Gray Nuclei  &Motor control, motor learning \citep{haber2010reward}\\ \hline 
34 & Pallidum & Subcortical Gray Nuclei  &Motor control \citep{kita2007globus}\\ \hline 
35 & Thalamus & Subcortical Gray Nuclei  &Sensory processing, motor control, cognitive functions \citep{sherman2002role}\\ \hline 
36 & Heschl gyrus & Temporal - Lateral  &Auditory processing, language processing \citep{smith2011morphometric}\\ \hline 
37 & Superior temporal gyrus & Temporal - Lateral  &Auditory processing, language processing, memory, Cognitive functions \citep{howard2000auditory}\\ \hline 
38 & Temporal pole & Temporal pole&Auditory processing, language processing, memory retrieval, episodic memory \citep{olson2007enigmatic,olson2013social}\\ \hline 
39 & Middle temporal gyrus & Temporal - Lateral  &Auditory processing, language processing, memory retrieval, episodic memory \citep{davey2016exploring}\\ \hline 
40 & Inferior temporal gyrus & Temporal - Lateral  &Visual processing, visual memory, cognitive functions \citep{grill2004human}\\ \hline 
41 & Cerebellum I II  & Cerebellum  &Motor control, coordination \citep{ito2008control}\\ \hline 
42 & Cerebellum III - VI & Cerebellum  &Motor control, coordination \citep{ito2008control}\\ \hline 
43 & Cerebellum VII - X & Cerebellum  &Motor control, coordination \citep{ito2008control}\\ \hline 
44 & Vermis & Cerebellum  &Motor control, coordination \citep{coffman2011cerebellar} \\ \hline

\end{tabular}
}
\label{tab:numbering}
\end{sidewaystable}

\begin{figure}[htb]
    \centering
    \begin{subfigure}[b]{.32\textwidth}
        \centering
         \includegraphics[width=\textwidth]{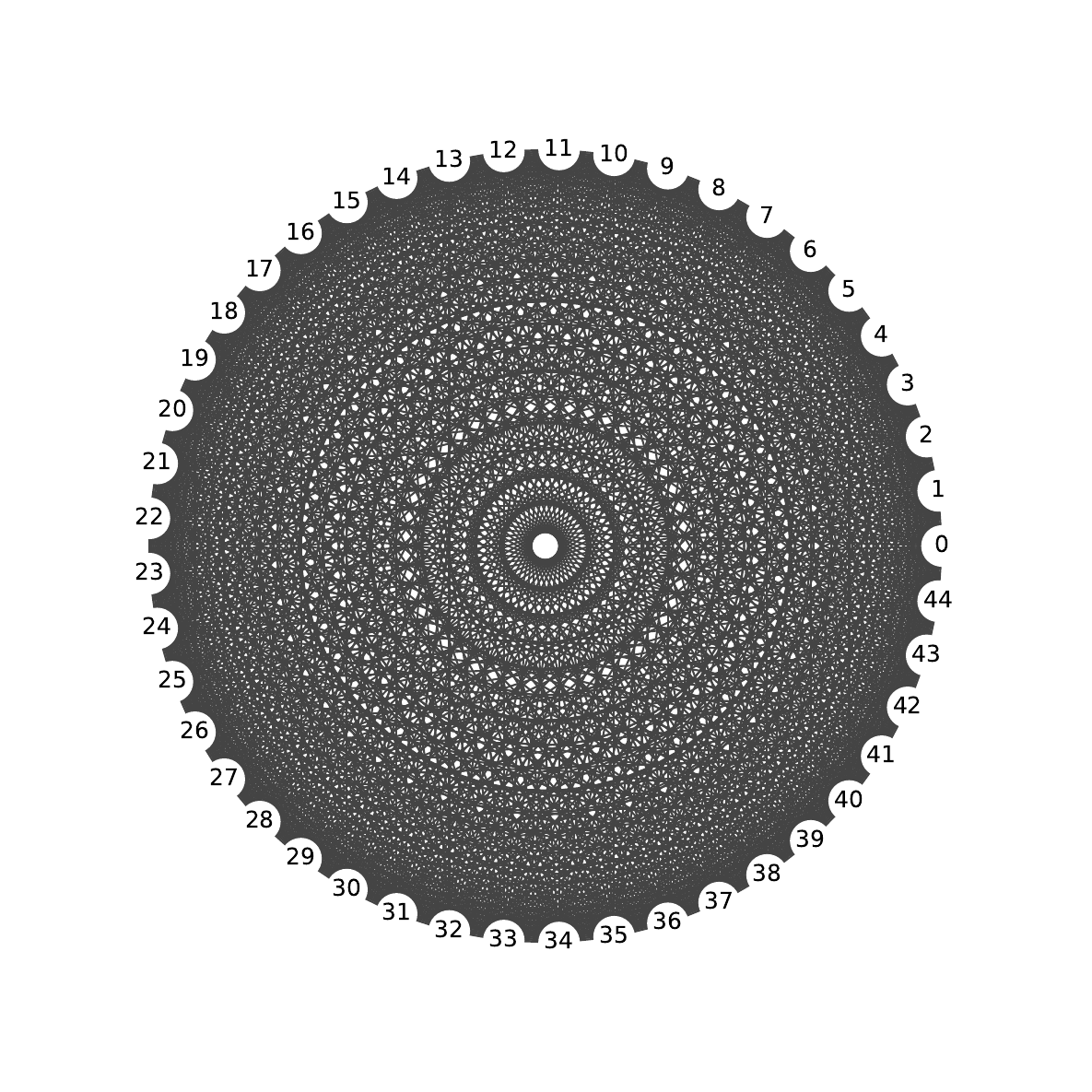}
         \caption{Pearson's correlation}
     \end{subfigure}
     \hfill
    \begin{subfigure}[b]{.32\textwidth}
        \centering
         \includegraphics[width=\textwidth]{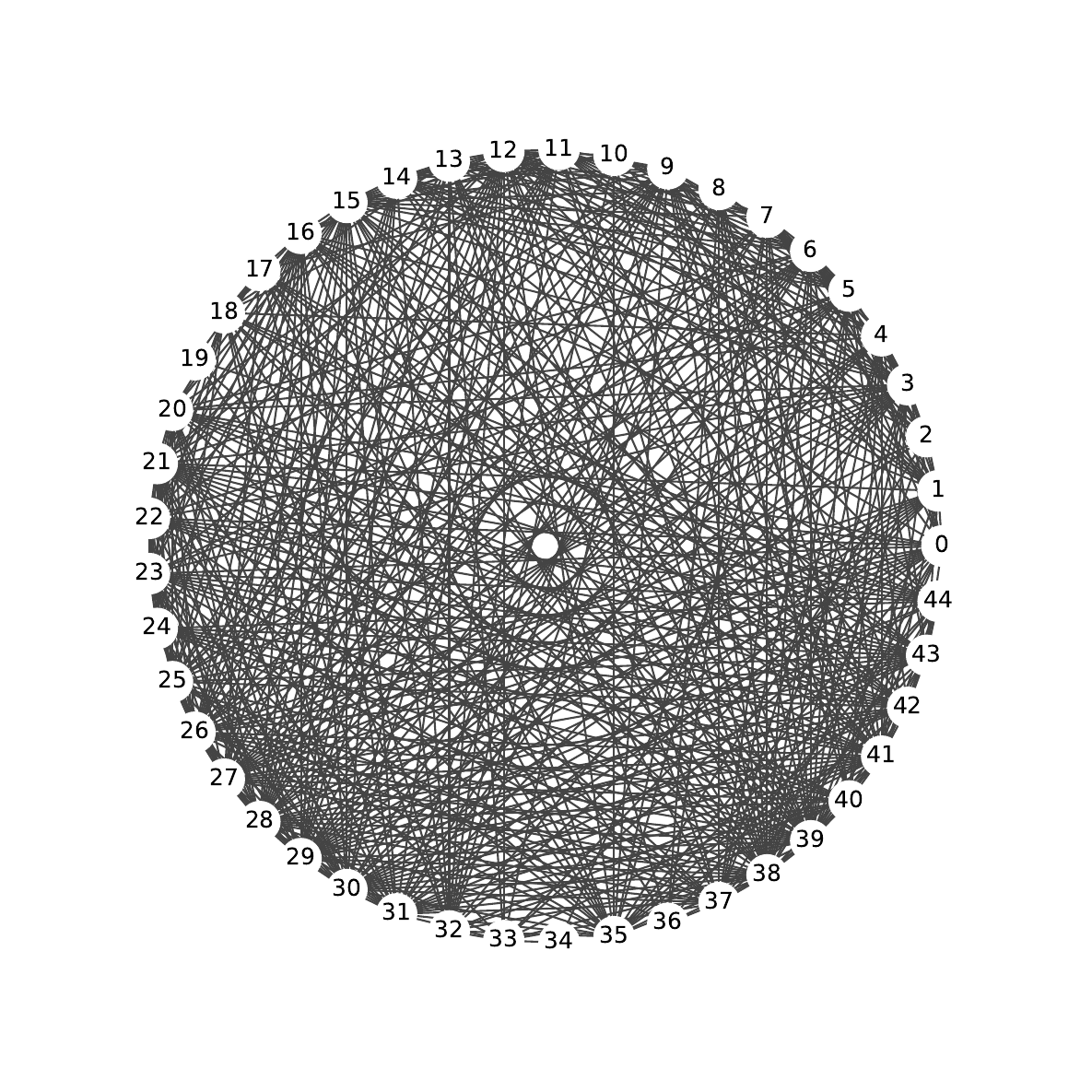}
         \caption{Partial correlation}
     \end{subfigure}
     \hfill
     \begin{subfigure}[b]{.32\textwidth}
        \centering
         \includegraphics[width=\textwidth]{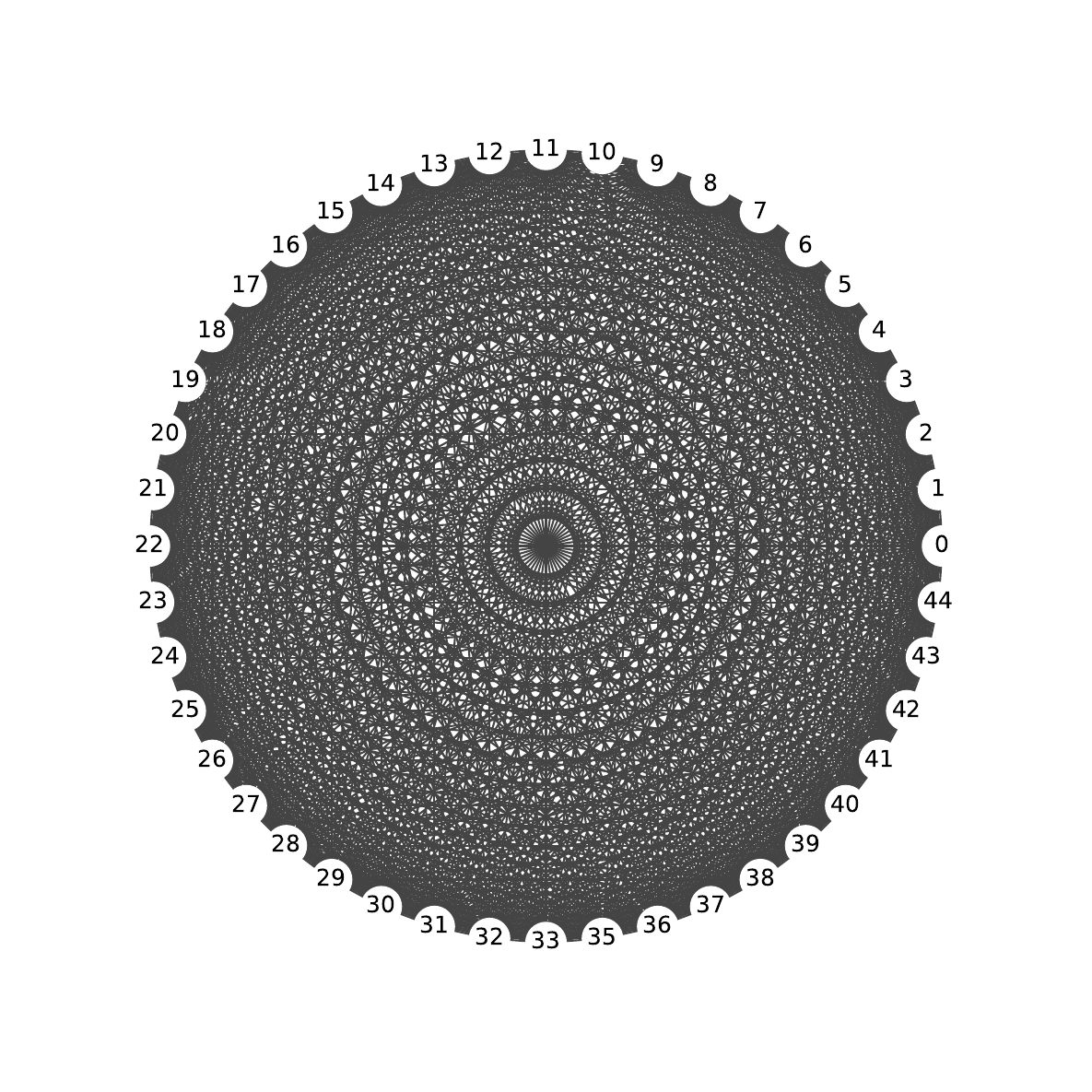}
         \caption{Mutual information}
     \end{subfigure}
     \vfill
     \hspace{15mm}
     \begin{subfigure}[b]{.32\textwidth}
        \centering
         \includegraphics[width=\textwidth]{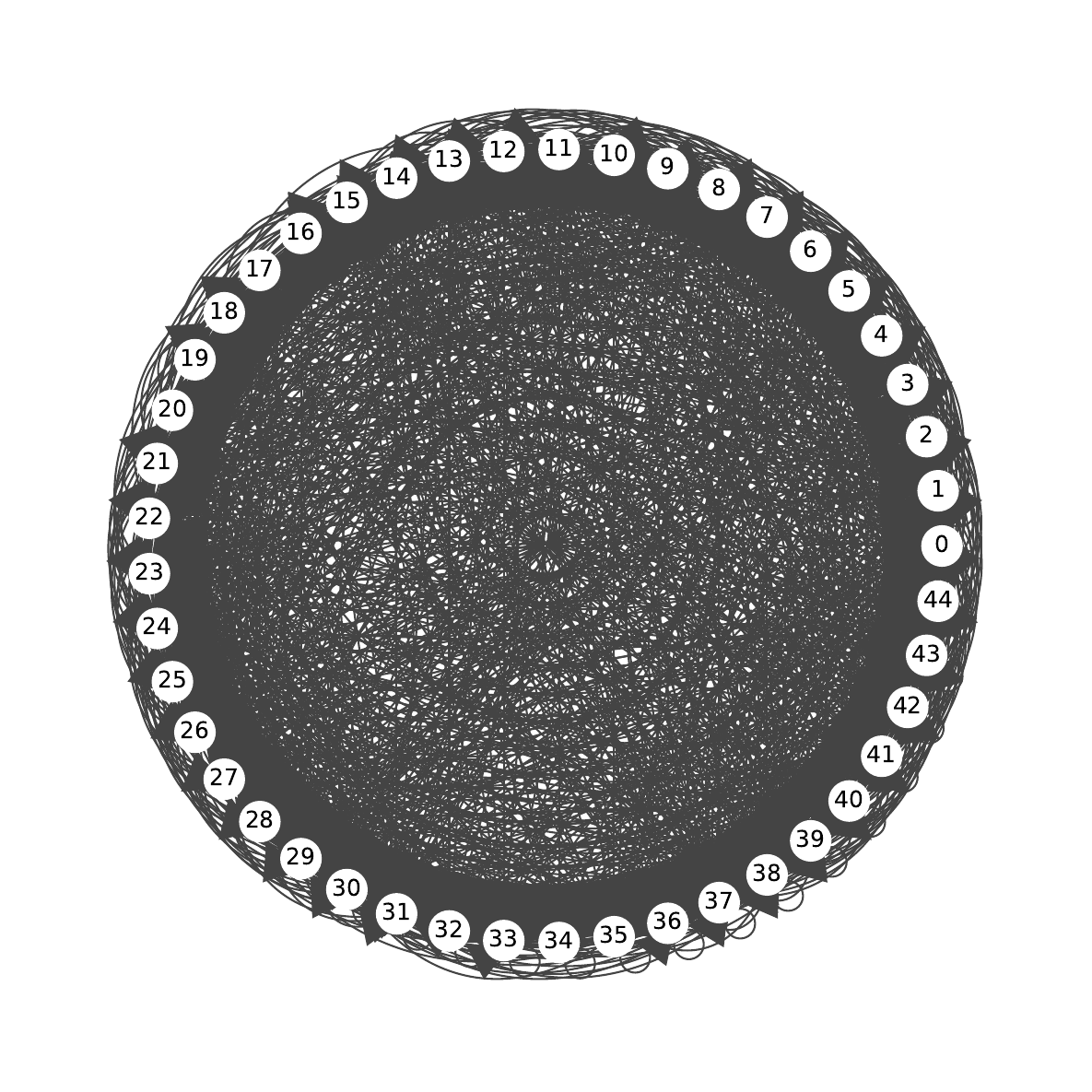}
         \caption{DTF}
     \end{subfigure}
     \hfill
     \begin{subfigure}[b]{.32\textwidth}
        \centering
         \includegraphics[width=\textwidth]{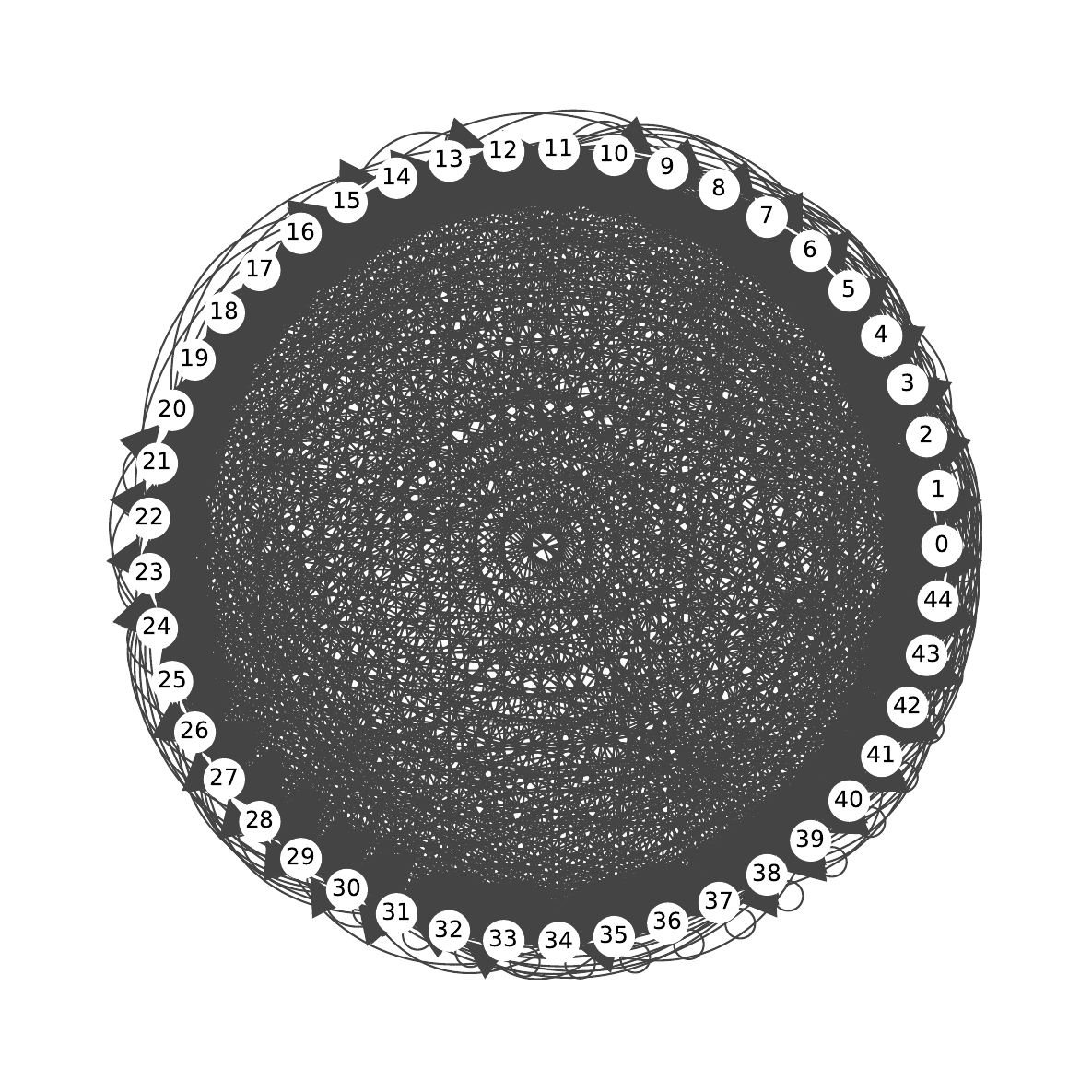}
         \caption{PDC}
     \end{subfigure}
     \hspace{15mm}
    \caption{Connectivity backbones retrieved by the considered baseline methods for the left hemisphere. ROIs numbering in \Cref{tab:numbering}.}
    \label{fig:baselinesl}
\end{figure}

\begin{figure}[htb]
    \centering
    \begin{subfigure}[b]{.32\textwidth}
        \centering
         \includegraphics[width=\textwidth]{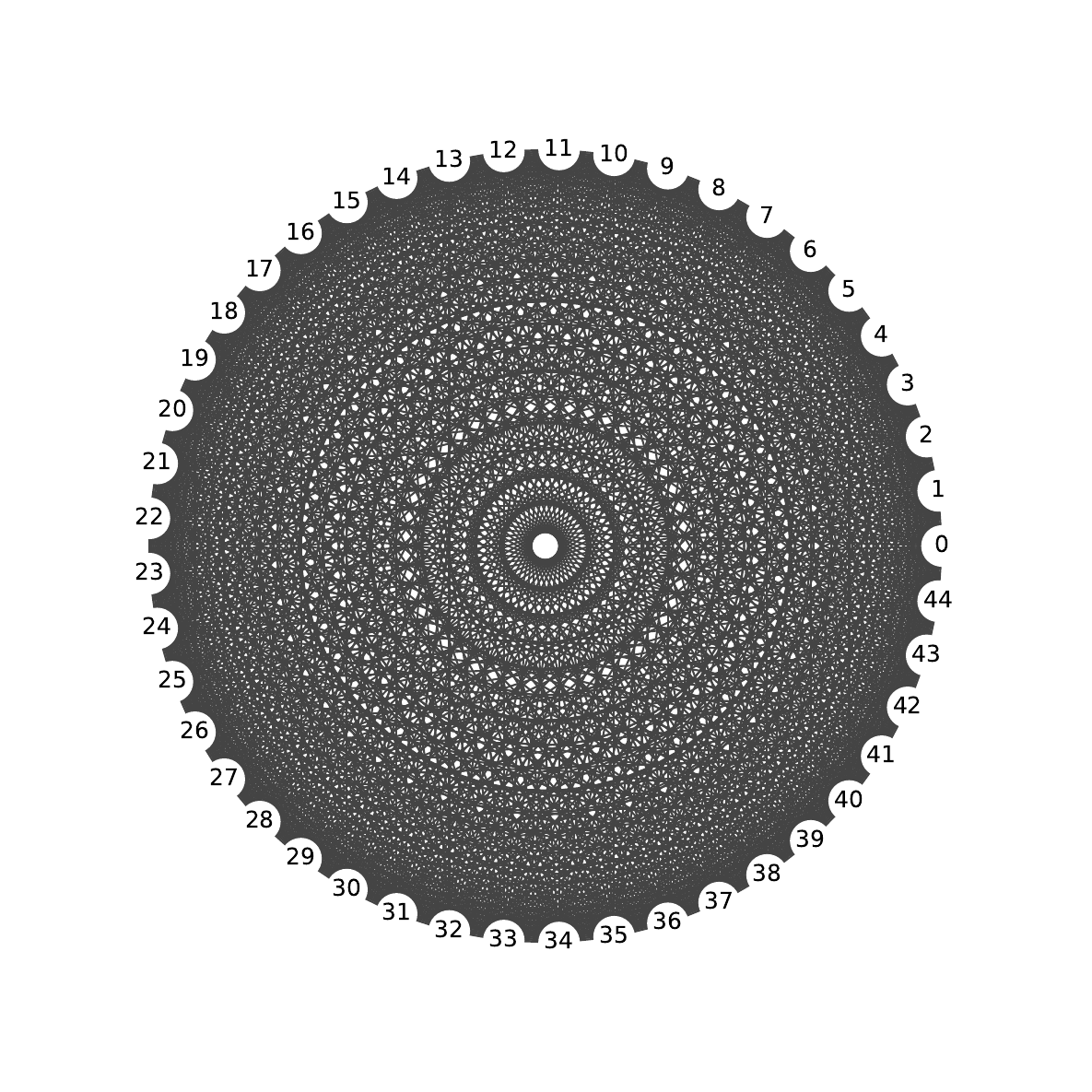}
         \caption{Pearson's correlation}
     \end{subfigure}
     \hfill
    \begin{subfigure}[b]{.32\textwidth}
        \centering
         \includegraphics[width=\textwidth]{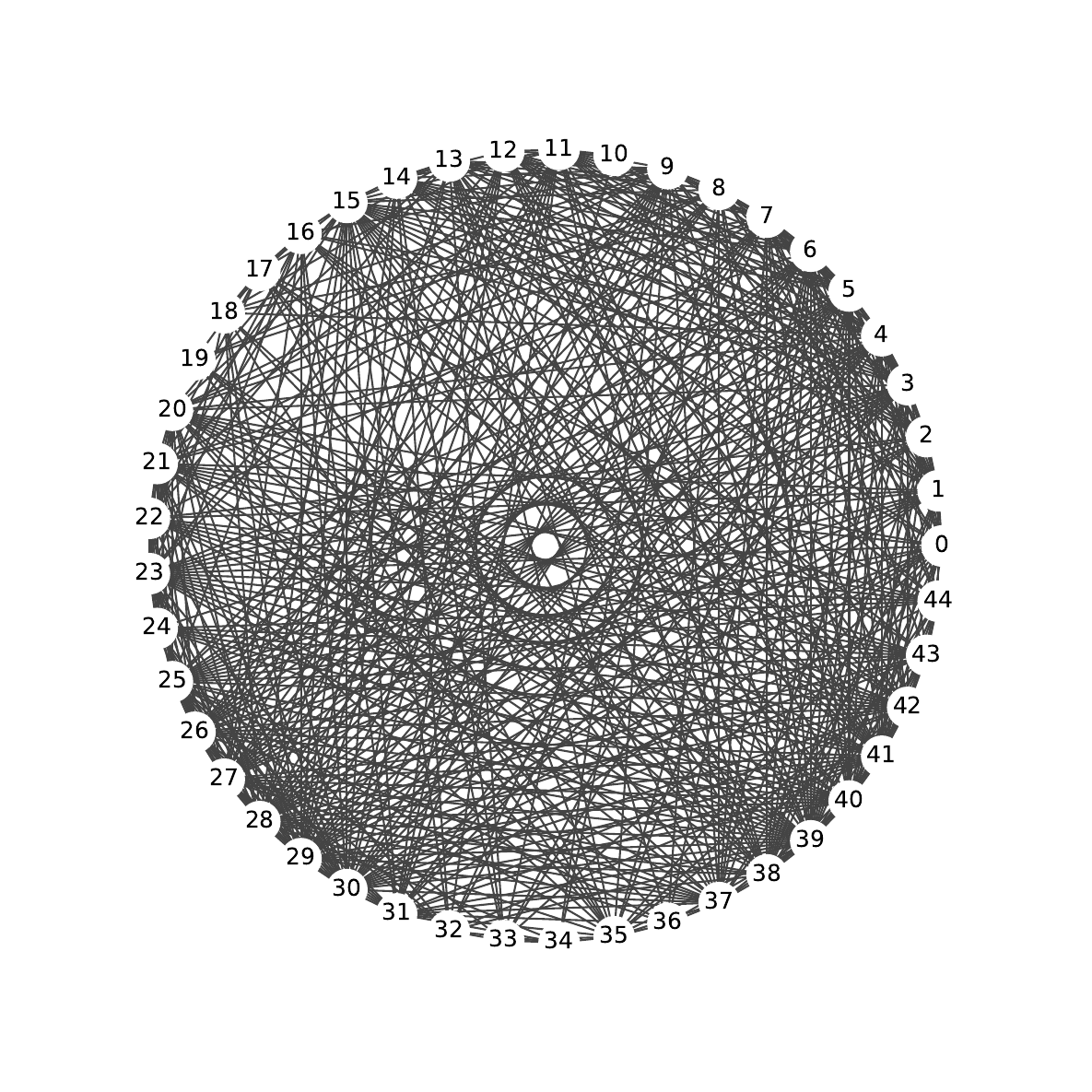}
         \caption{Partial correlation}
     \end{subfigure}
     \hfill
     \begin{subfigure}[b]{.32\textwidth}
        \centering
         \includegraphics[width=\textwidth]{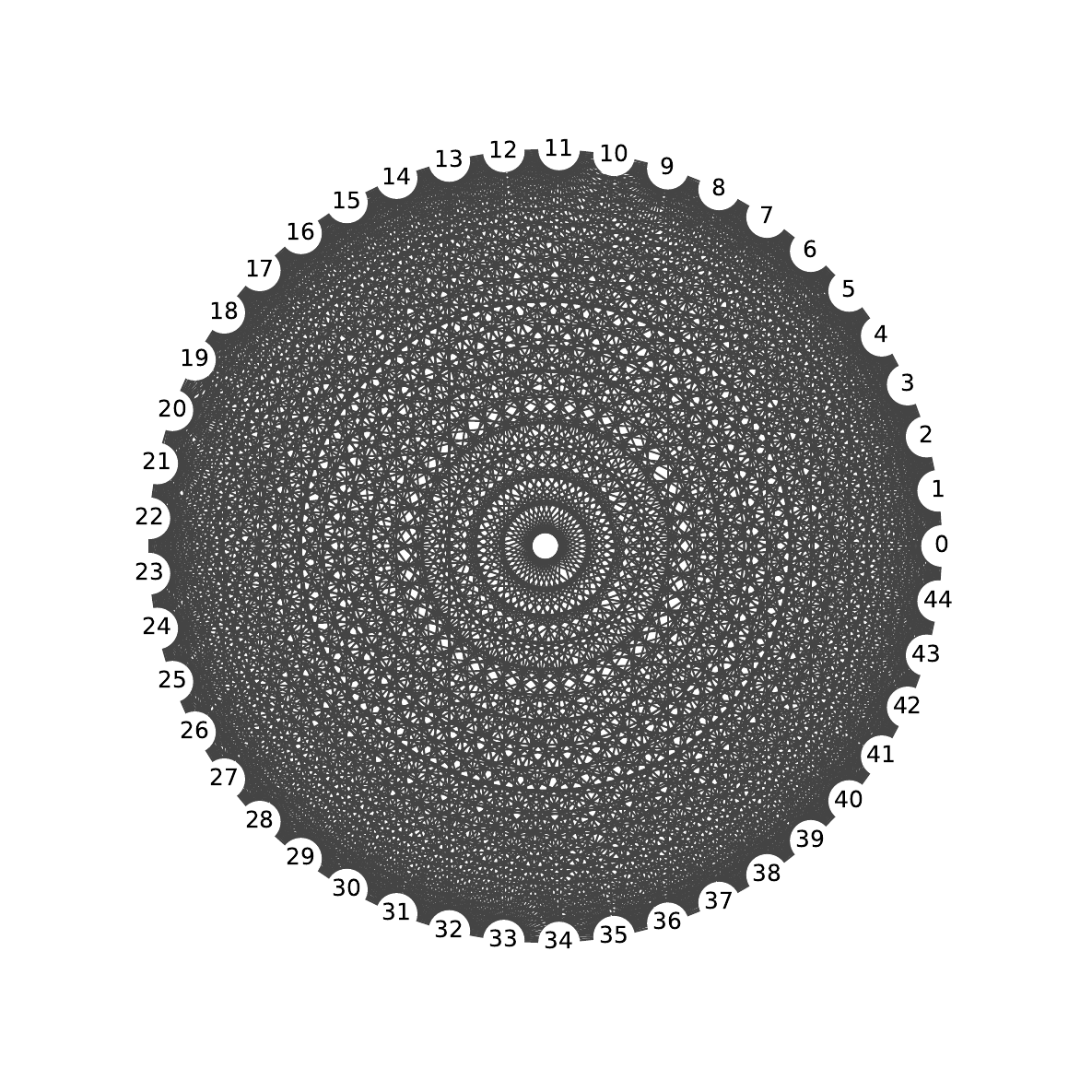}
         \caption{Mutual information}
     \end{subfigure}
     \vfill
     \hspace{15mm}
     \begin{subfigure}[b]{.32\textwidth}
        \centering
         \includegraphics[width=\textwidth]{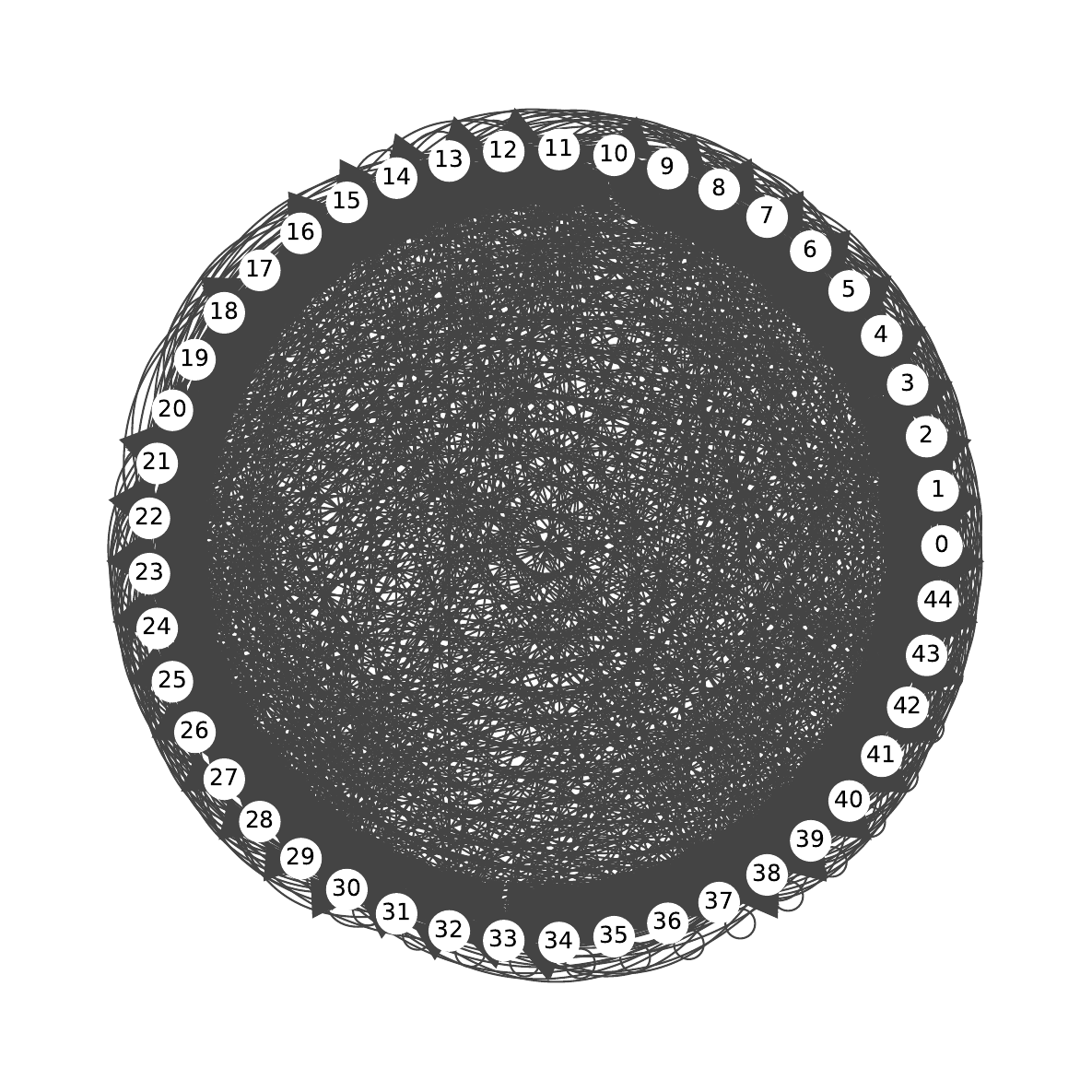}
         \caption{DTF}
     \end{subfigure}
     \hfill
     \begin{subfigure}[b]{.32\textwidth}
        \centering
         \includegraphics[width=\textwidth]{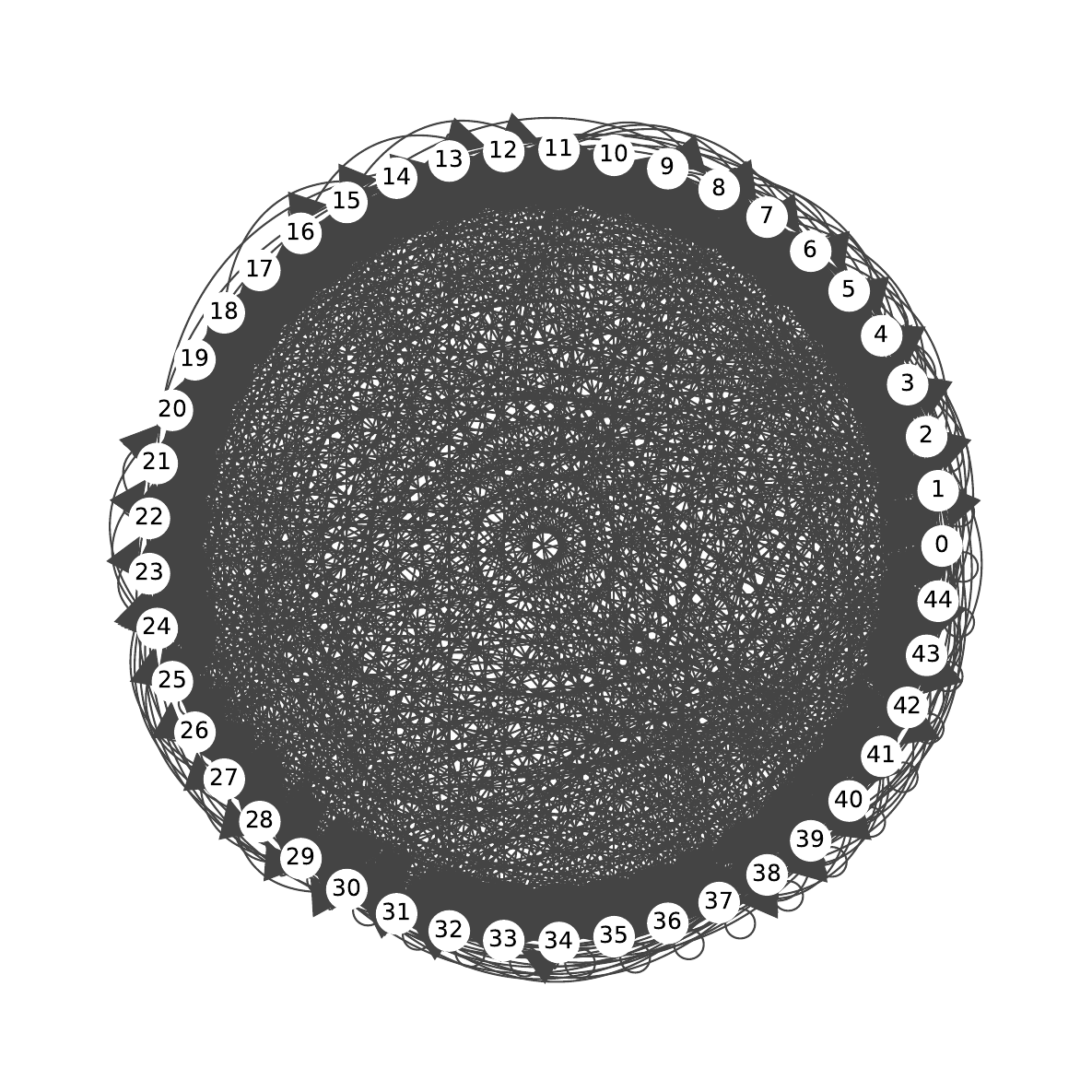}
         \caption{PDC}
     \end{subfigure}
     \hspace{15mm}
    \caption{Connectivity backbones retrieved by the considered baseline methods for the right hemisphere. ROIs numbering in \Cref{tab:numbering}.}
    \label{fig:baselinesr}
\end{figure}

\Cref{fig:baselinesl,fig:baselinesr} show that the baseline methods retrieve dense networks for both the right and left hemispheres.
Consequently, these methods suggest that the common connectivity matrix involves all the ROIs, in an intricate fashion.

\begin{figure}[htb]
    \centering
    \begin{minipage}[b]{.25\textwidth}
        \centering
         \includegraphics[width=1.0\textwidth]{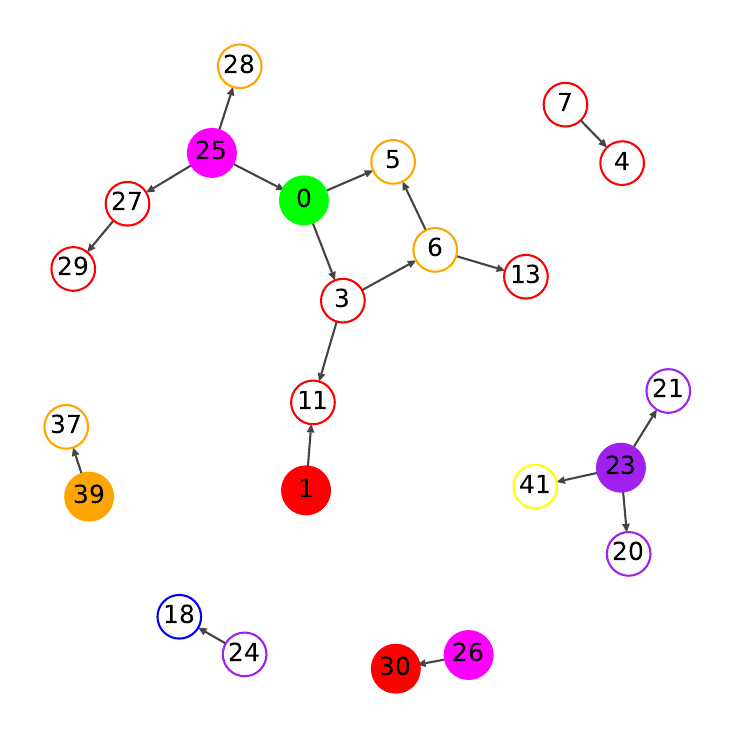}
         \subcaption{Scale $1$, left}
         \label{subfig:left1}
         % \vspace{-8mm}
     \end{minipage}
     \hfill
    \begin{minipage}[b]{.25\textwidth}
        \centering
         \includegraphics[width=1.0\textwidth]{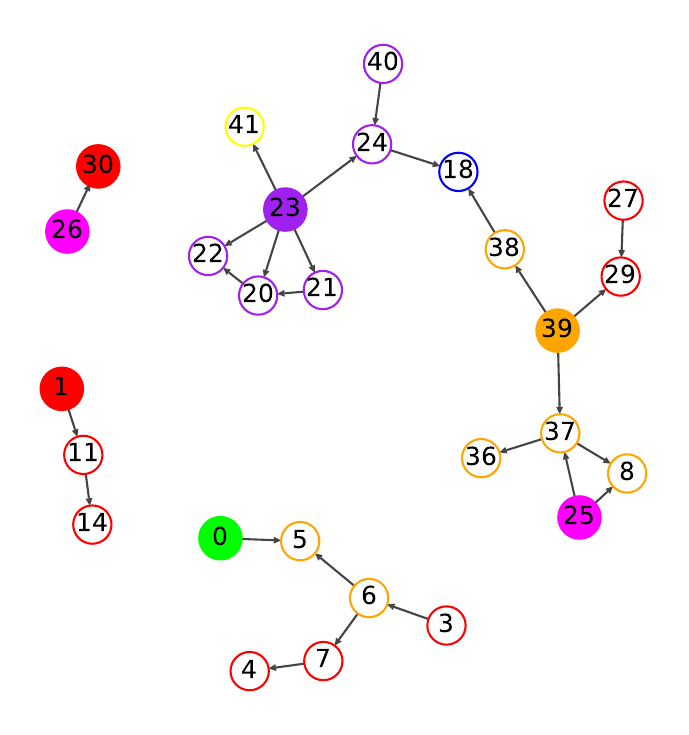}
         \subcaption{Scale $2$, left}
         \label{subfig:left2}
     \end{minipage}
     \hfill
     \begin{minipage}[b]{.25\textwidth}
        \centering
         \includegraphics[width=1.0\textwidth]{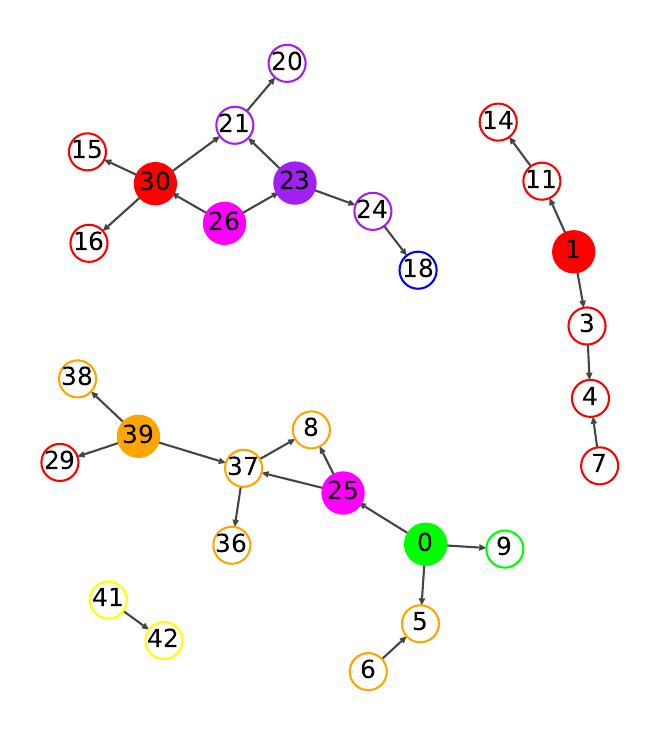}
         \subcaption{Scale $3$, left}
         \label{subfig:left3}
         % \vspace{-8mm}
     \end{minipage}
     \vfill
     \hspace{15mm}
     \begin{minipage}[b]{.25\textwidth}
        \centering
         \includegraphics[width=1.0\textwidth]{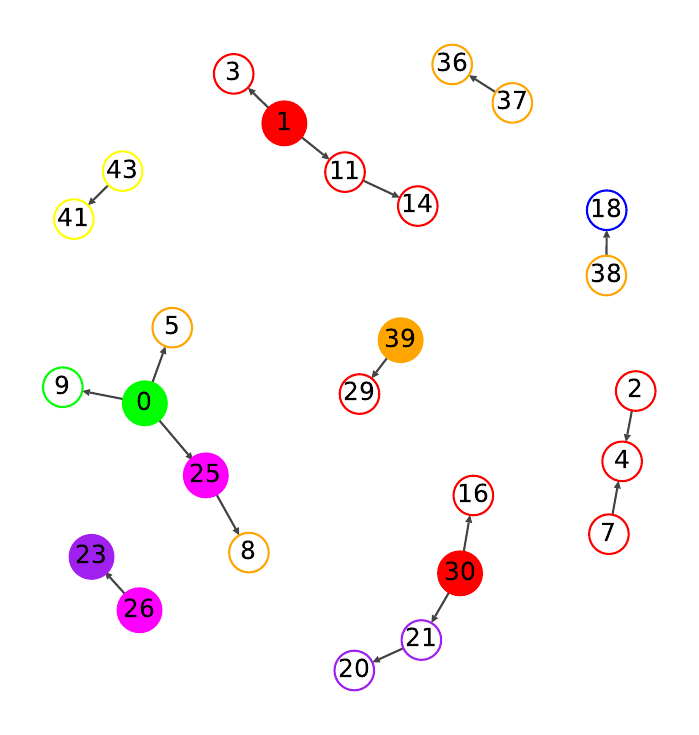}
         \subcaption{Scale $4$, left}
         \label{subfig:left4}
     \end{minipage}
     \hfill
     \begin{minipage}[b]{.25\textwidth}
        \centering
         \includegraphics[width=1.0\textwidth]{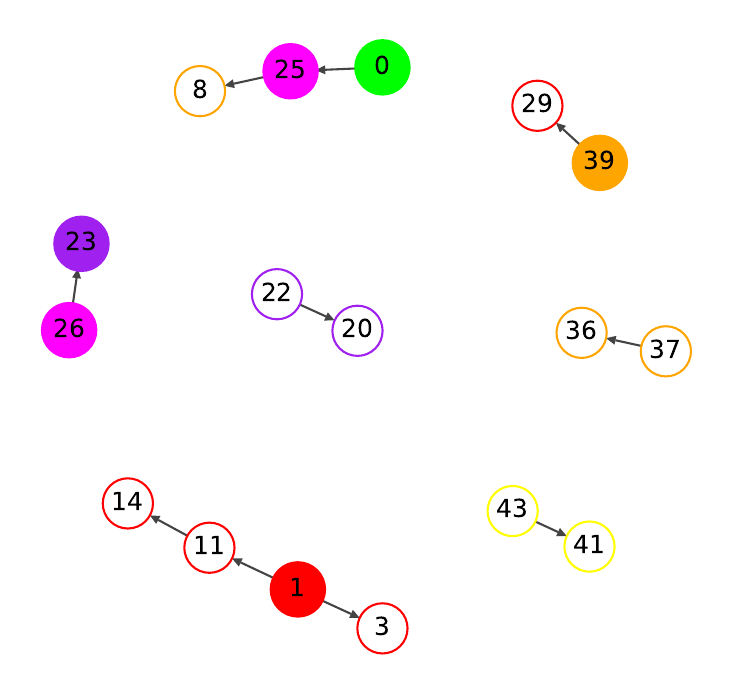}
         \subcaption{Scale $5$, left}
         \label{subfig:left5}
         % \vspace{-8mm}
     \end{minipage}
     \hspace{15mm}
     \vfill
     \begin{minipage}[b]{.25\textwidth}
        \centering
         \includegraphics[width=1.0\textwidth]{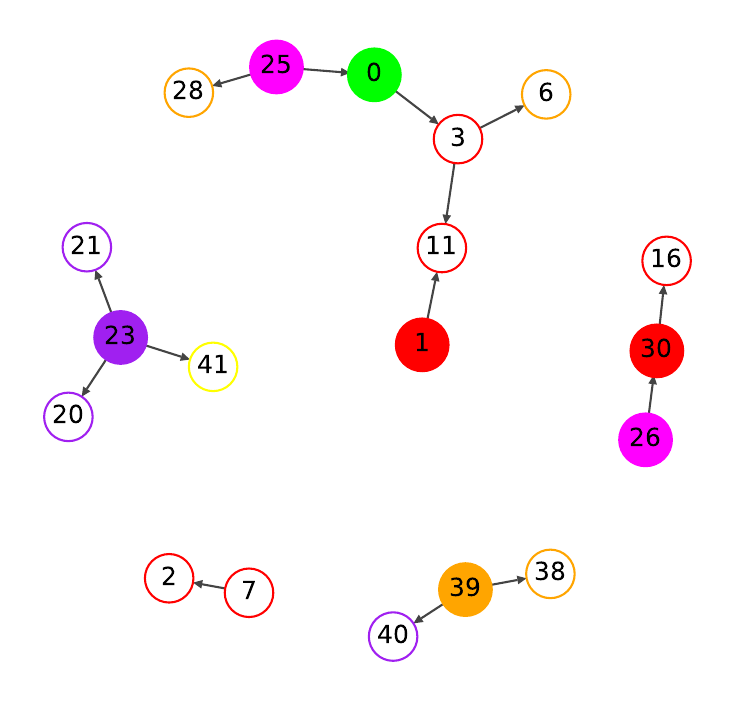}
         \subcaption{Scale $1$, right}
         % \vspace{3mm}
         \label{subfig:right1}
     \end{minipage}
     \hfill
    \begin{minipage}[b]{.25\textwidth}
        \centering
         \includegraphics[width=1.0\textwidth]{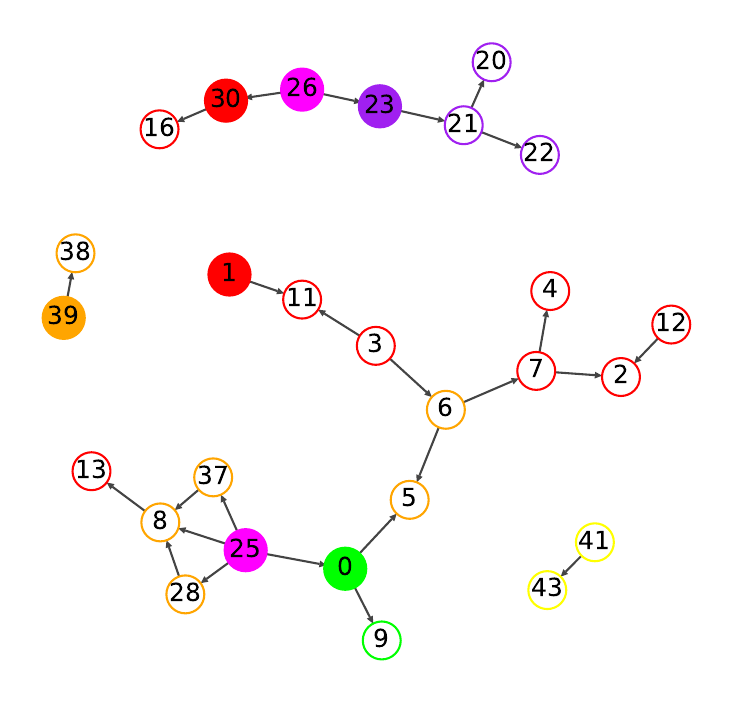}
         \subcaption{Scale $2$, right}
         \label{subfig:right2}
     \end{minipage}
     \hfill
     \begin{minipage}[b]{.25\textwidth}
        \centering
         \includegraphics[width=1.0\textwidth]{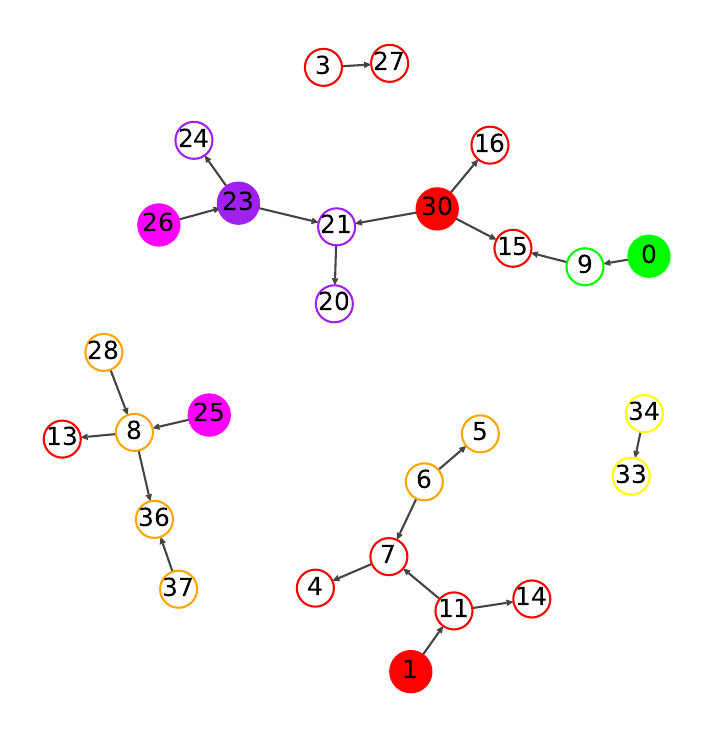}
         \subcaption{Scale $3$, right}
         % \vspace{3mm}
         \label{subfig:right3}
     \end{minipage}
     \vfill
     \hspace{15mm}
     \begin{minipage}[b]{.25\textwidth}
        \centering
         \includegraphics[width=1.0\textwidth]{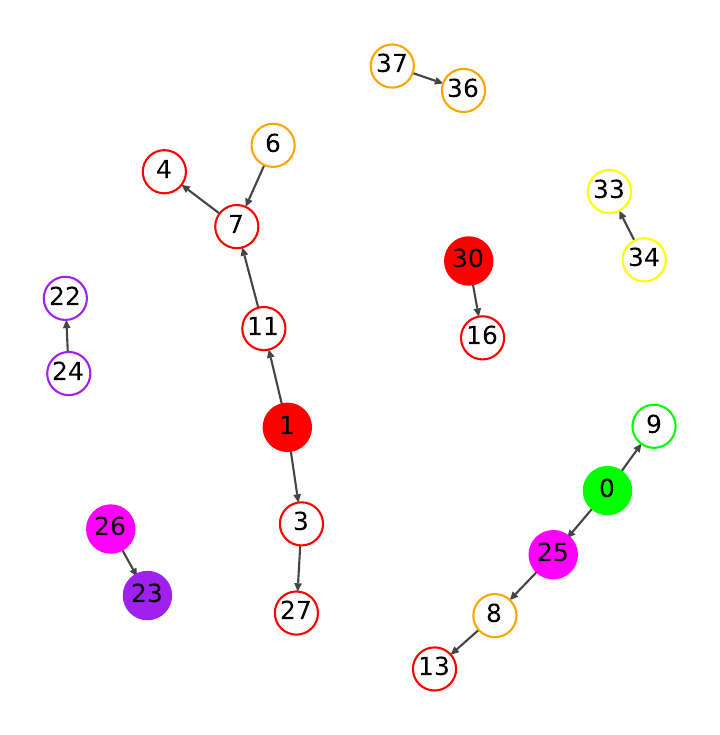}
         \subcaption{Scale $4$, right}
         \label{subfig:right4}
     \end{minipage}
     \hfill
     \begin{minipage}[b]{.25\textwidth}
        \centering
         \includegraphics[width=1.0\textwidth]{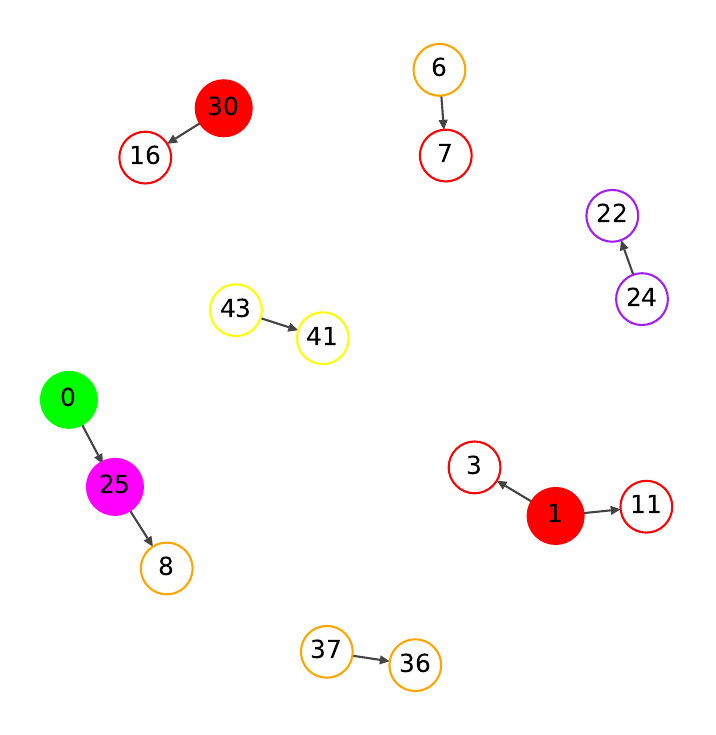}
         \subcaption{Scale $5$, right}
         % \vspace{3mm}
         \label{subfig:right5}
     \end{minipage}
     \hspace{15mm}
     % \vspace{-5mm}
    \caption{MCBs for the left and the right hemispheres. The main driver nodes are color-filled. ROIs numbering in \Cref{tab:numbering}.}
    \label{fig:MCBsfull}
    % \vspace{-5mm}
\end{figure}

Conversely, our method proposes sparse MCBs, which allow to extract meaningful information from data.
\Cref{fig:MCBsfull} depicts the obtained results for the left (depicted on top) and the right hemispheres (bottom). 
We plot separately the causal backbones for distinct time scales, from the finest on the left to the coarsest on the right.
Additionally, to better visualize the results, we use the following color coding inspired by the ROIs functions given in \Cref{tab:numbering} :
\begin{squishlisttight}
    \item \textcolor{red}{Red} for ROIs corresponding to cognitive functions, attention, emotion, and decision-making;
    \item \textcolor{orange}{Orange} for those related to auditory processing, speech and language processing, and memory;
    \item \textcolor{blue}{Blue} for those concerning memory formation and memory retrieval;
    \item \textcolor{magenta}{Pink} for those associated with sensory integration and somatosensory;
    \item \textcolor{violet}{Purple} for the ROIs within the visual network and related to the visual memory;
    \item \textcolor{applegreen}{Green} for those within the motor network;
    \item \textcolor{amber}{Yellow} for those regarding the motor control and the posture.
\end{squishlisttight}
We provide below a detailed discussion of the obtained results, organized by scales.

\spara{Scale 1 $([0.694-0.347]\mathrm{Hz}).$} 
Looking at the MCBs in \Cref{subfig:left1,subfig:right1}, corresponding to the finest scale, we observe a central role of the occipital lobe ($23$) within the visual network (VN), for both hemispheres.
It is also interesting to see the impact of the superior parietal gyrus ($26$), linked to sensory integration and attention (\citealp{corbetta2002control,culham2001neuroimaging} and references therein), on the precuneus ($30$), central node of the default mode network (DMN, \citealp{raichle2001default, buckner2008brain}). 
This means that the integration of sensory information from different modalities, such as proprioception and vision, activates the self-awareness and self-referential processing of $30$. 
Continuing on this line, we see that the postcentral gyrus ($25$), primarily involved in processing somatosensory information \citep{iwamura1994bilateral, disbrow2000somatotopic}, plays an important role.
Specifically, it has an impact on the precentral gyrus ($0$), responsible for the planning and execution of voluntary motor movements \citep{penfield1937somatic, rizzolatti2001cortical}, which then activate the middle frontal gyrus ($3$), involved in higher-order cognitive processes, such as decision-making \cite{ridderinkhof2004role,koechlin2003architecture}.
Furthermore, $25$ is shown to cause the activity of the supramarginal gyrus ($28$), which is involved in sensorimotor integration and helps to bridge the gap between sensory perception and motor action \citep{rushworth2003left}.
We also notice, on the left hemisphere, an impact on the parietal inferior gyrus ($27$), involved in cognitive processes \citep{caspers2006human}.
We attribute to this connection the same meaning as $26 \rightarrow 30$.
In addition, we see that the inferior frontal gyrus ($7$), associated in both hemispheres to emotion processing (\citealp{rolls2000orbitofrontal} and references therein), impacts the middle frontal ($4$) and superior frontal ($2$) orbital gyri, involved in emotion regulation (\citealp{zhao2020reduced} and references therein) and cognitive functions \citep{hu2016right}, respectively.
Finally, the middle temporal gyrus ($39$), which is involved in language and visual processing but also in episodic memory and higher-level cognition \citep{davey2016exploring,briggs2021unique}, is reported to influence the activity of the superior temporal gyrus ($37$), related to auditory and language-related functions \citep{howard2000auditory}, and the temporal pole ($38$) which is linked to emotion, memory and social cognition \citep{olson2007enigmatic,olson2013social}.

\spara{Scale 2 $([0.347-0.174]\mathrm{Hz}).$}
Looking at the MCBs in \Cref{subfig:left2,subfig:right2}, corresponding to the second finest scale, we still observe the key role of the $23$ within the VN, for both hemispheres.
Interestingly, we have connections from $26$ to nodes belonging to the DMN, but also a causal link toward $23$ in the VN, thus acting as a common cause.
Hence, here the integration of sensory information impacts visual processing as well. 
Furthermore, we see that the processing of somatosensory information done by $25$ is a common cause for the planning and execution of voluntary motor movements, sensorimotor integration, and auditory and language-related functions.
Among the red nodes, we see again the central role of $7$ within the orbital surface of the frontal lobe.
Also, it is interesting to see that the triangular part of the inferior frontal gyrus ($6$) mediates the effect of $3$ on $7$. 
This is consistent with the fact that the three nodes belong to the frontal executive network, which enables cognitive flexibility, decision-making, planning, inhibition of irrelevant information, and goal-directed behavior.
Additionally, the superior frontal gyrus (lateral surface, $1$), involved in cognitive functions and attention \citep{boisgueheneuc2006functions}, impacts the superior frontal gyrus (medial surface, $11$), which belongs to the DMN and thus relates to self-awareness and introspection.
Finally, $39$ tends to cause the other regions within the temporal lobe.

\spara{Scale 3 $([0.174-0.087]\mathrm{Hz}).$}
Referencing \Cref{subfig:left3,subfig:right3}, we see that also at this scale we can draw similar conclusions regarding the position of $23$ within the VN, and the common cause role of $25$ and $26$ related to sensory integration.
Another analogy with the previous scales concerns nodes $1$ and $7$, which continue to cause the DMN activity and executive control functions.
In addition, the causal graph corresponding to the left hemisphere highlights again the importance of $39$ within the temporal lobe, and its impact on the angular gyrus, part of the DMN.
However, two distinctions from the previous scales emerge.
The first concerns node $30$, which is a common cause for nodes in the DMN and the cuneus ($21$) of the VN.
Specifically, $21$ processes visual information received from $23$, as also shown in the previous two scales.
Hence, the precuneus provides context (e.g., episodic memory retrieval) to the cuneus, influencing how visual information is interpreted during mind-wandering.
The second regards the relations between nodes $0$ and $25$, belonging to the sensory-motor network \citep{mantini2007electrophysiological,smith2009correspondence}, which is now reversed.
The bi-directionality of this interaction confirms that the two regions collaborate to control and sense movements.
Node $0$ sends signals to initiate the movement, and $25$ receives sensory feedback about the position and state of the body during the action.
Then, this information is used to calibrate the subsequent voluntary action.
Finally, in the left hemisphere, the central role of $39$ within the temporal lobe is further highlighted.

\spara{Scale 4 $([0.087-0.043]\mathrm{Hz}).$}
The analysis of \Cref{subfig:left4,subfig:right4} leads to conclusions similar to scale $3$. 
Nodes $1$ and $30$ are confirmed to be key drivers for high-level cognitive processes and visual processing.
The impact of node $0$ on $25$ is reported, and the connections $26 \rightarrow 23$, $39 \rightarrow 29$ persist at this scale as well.

\spara{Scale 5 $([0.043-0.022]\mathrm{Hz}).$}
Looking at \Cref{subfig:left5,subfig:right5}, we see that here the MCBs are sparser.
While the V-structure $11\leftarrow 1 \rightarrow 3$ and the chain $0 \rightarrow 25 \rightarrow 8$ persist in both hemispheres, the connections $30\rightarrow 16$ and $39 \rightarrow 29$ appear only on the left one.
In addition, the interaction between $30$ and $21$ disappears at this frequency band.
\clearpage

\subsection{Graphical representation of the SCBs}\label{subapp:addresults4}

\begin{figure}[htb]
    \hspace{15mm}
    \centering
    \begin{minipage}[b]{.35\textwidth}
        \centering
         \includegraphics[width=1.0\textwidth]{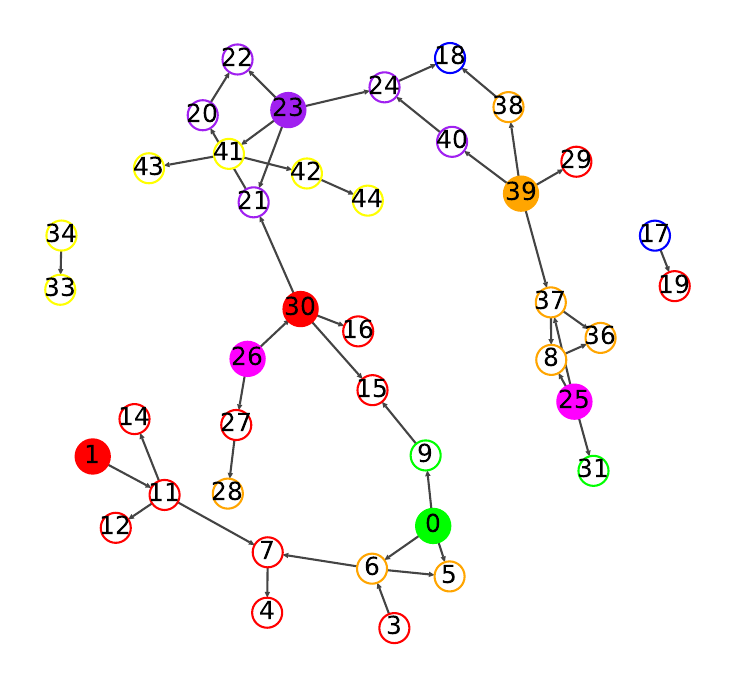}
         \subcaption{Left}
         \label{subfig:ssleft}
         \vspace{3mm}
     \end{minipage}
     \vspace{-2mm}
     \hfill
     \begin{minipage}[b]{.35\textwidth}
        \centering
         \includegraphics[width=1.0\textwidth]{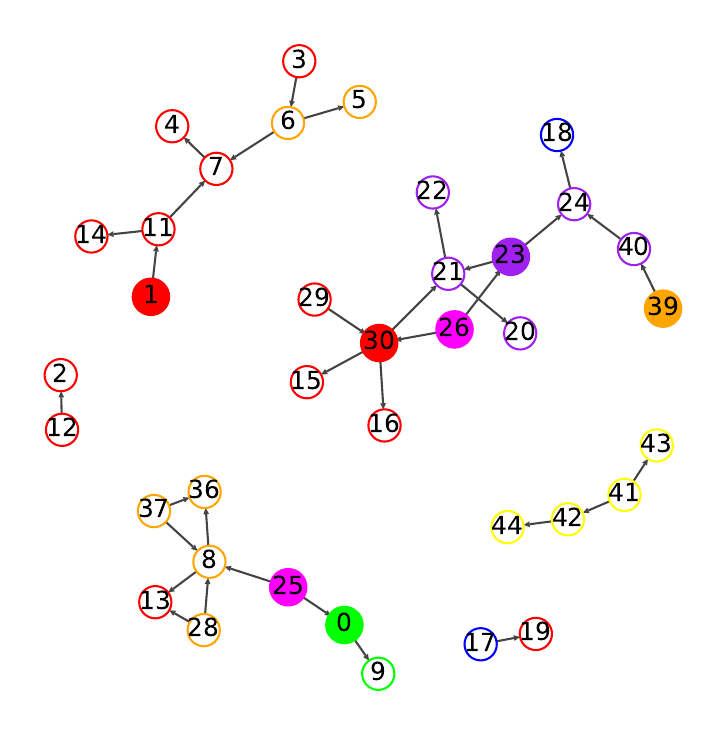}
         \subcaption{Right}
         \vspace{3mm}
         \label{subfig:ssright}
     \end{minipage}
     \hspace{15mm}
     \vspace{-2mm}
    \caption{SCBs for (\subref{subfig:ssleft}) the left and (\subref{subfig:ssright}) the right hemispheres. Color-filled nodes are the main drivers within the MCBs in \Cref{fig:MCBsfull}.}
    \label{fig:SCBs}
\end{figure}

In this appendix we provide the plots pertaining to the SCBs discussed in \Cref{sec:results}.
In order to facilitate the comparison with the MCBs in \Cref{fig:MCBsfull} we use the same numbering and color-coding for the ROIs.

\end{document}